\documentclass[final,12pt]{colt2026} 

\usepackage{enumitem}
\usepackage{booktabs}
\usepackage{makecell}
\usepackage[table]{xcolor}
\usepackage{graphicx}
\graphicspath{ {./imgs/} }
\usepackage{subcaption}
\usepackage{thmtools}

\usepackage{bm}

\def\vzero{{\bm{0}}}
\def\ve{{\bm{e}}}

\def\vv{{\bm{v}}}
\def\vw{{\bm{w}}}
\def\vx{{\bm{x}}}

\def\mI{{\bm{I}}}

\DeclareMathOperator*{\argmax}{argmax}
\DeclareMathOperator*{\argmin}{argmin}

\def\Eqref#1{Equation~(\ref{#1})}
\newcommand{\Eqmark}[2]{\stackrel{(#1)}{#2}}



\newtheorem{theorem}{Theorem}[section]
\newtheorem{lemma}{Lemma}[section]
\newtheorem{assumption}{Assumption}[section]
\newtheorem{definition}{Definition}[section]
\newtheorem{proposition}{Proposition}[section]
\newtheorem{corollary}{Corollary}[section]
\newtheorem{remark}{Remark}[section]

\title[Large Stepsize Gradient Descent for Logistic Regression in Low Dimension]{Tight Bounds for Logistic Regression with Large Stepsize Gradient Descent in Low Dimension}
\usepackage{times}

\coltauthor{
 \Name{Michael Crawshaw} \Email{mcrawsha@gmu.edu}\\
 \Name{Mingrui Liu} \Email{mingruil@gmu.edu}\\
 \addr Department of Computer Science, George Mason University, Fairfax, VA 22030
}

\begin{document}

\maketitle

\begin{abstract}
We consider the optimization problem of minimizing the logistic loss with gradient descent to train a linear model for binary classification with separable data. With a budget of $T$ iterations, it was recently shown that an accelerated $1/T^2$ rate is possible by choosing a large stepsize $\eta = \Theta(\gamma^2 T)$ (where $\gamma$ is the dataset's margin) despite the resulting non-monotonicity of the loss. In this paper, we provide a tighter analysis of gradient descent for this problem when the data is two-dimensional: we show that GD with a sufficiently large learning rate $\eta$ finds a point with loss smaller than $\mathcal{O}(1/(\eta \gamma^2 T))$, as long as $T \geq \Omega(n/\gamma + 1/\gamma^2)$, where $n$ is the dataset size. Our improved rate comes from a tighter bound on the time $\tau$ that it takes for GD to transition from unstable (non-monotonic loss) to stable (monotonic loss), via a fine-grained analysis of the oscillatory dynamics of GD in the subspace orthogonal to the max-margin classifier. We also provide a lower bound of $\tau$ matching our upper bound up to logarithmic factors, showing that our analysis is tight.
\end{abstract}


\section{Introduction}
In modern machine learning, optimization algorithms tend to operate in ``unstable"
regimes, where the loss does not monotonically decrease over time, even with full-batch
gradients \citep{cohen2021gradient}. However, the theory of optimization for machine
learning largely considers only stable regimes, where sufficiently small stepsizes for
gradient descent (GD) and its variants will safely ensure monotonic loss decrease
\citep{nesterov2013introductory}. Our limited understanding of unstable optimization has
created a significant gap between optimization algorithms that work, and optimization
algorithms that we understand theoretically.

In this paper, we consider the problem of training linear models for binary
classification by minimizing the logistic loss using gradient descent with large
stepsizes. Despite this problem's simplicity and its fundamental role in machine
learning, the optimization dynamics of gradient descent with large step sizes in this
setting is still not entirely understood. We aim to provide upper and lower bounds on
the iteration complexity required for gradient descent to find an approximate solution
to the optimization problem in question.

Given a dataset of $n$ samples $\{(\vx_i, y_i)\}_{i=1}^n$, where each sample consists of
an input $\vx_i \in \mathbb{R}^d$ and a label $y_i \in \{-1, 1\}$, we denote $\ell(z) =
\log(1 + \exp(-z))$ and consider the optimization problem:
\begin{equation} \label{eq:opt_prob}
    \min_{\vw \in \mathbb{R}^d} \left\{ F(\vw) := \frac{1}{n} \sum_{i=1}^n \ell(y_i \langle \vw, \vx_i \rangle) \right\}.
\end{equation}

We make the following assumptions on the dataset.

~\\
\begin{assumption} \label{ass:dataset}
\begin{enumerate}[label=(\alph*)]
    \item Linear separability: There exists $\vw \in \mathbb{R}^d$ such that
        $y_i \langle \vw, \vx_i \rangle > 0$ for all $i \in [n]$.
    \item Bounded data norm: $\|\vx_i\| \leq 1$ for all $i \in [n]$.
    \item Identical label: $y_i = 1$ for all $i \in [n]$.
\end{enumerate}
\end{assumption}

Assumption \ref{ass:dataset}(b) can be satisfied in practice by replacing all data
$\vx_i$ with $\vx_i / \max_{j \in [n]} \|\vx_j\|$, and Assumption \ref{ass:dataset}(c)
can be made without loss of generality: the objective $F$ only depends on the label
$y_i$ through the product $y_i \vx_i$, so that any sample $(\vx_i, -1)$ can be replaced
by $(-\vx_i, 1)$ while preserving the objective. The same assumptions are made in
previous work \citep{wu2023implicit, wu2024large}.

For a given dataset, we define the maximum margin and the associated classifier by
\begin{equation}
    \gamma = \max_{\|\vw\| = 1} \min_{i \in [n]} y_i \langle \vw, \vx_i \rangle, \quad \vw_* = \argmax_{\|\vw\| = 1} \min_{i \in [n]} y_i \langle \vw, \vx_i \rangle.
\end{equation}
Note $0 < \gamma < 1$, since separable data implies $\gamma > 0$ and bounded data
implies $\gamma \leq \|\vx_i\| \leq 1$.

We consider Gradient Descent (GD) with a constant stepsize $\eta > 0$ for minimizing
\Eqref{eq:opt_prob}:
\begin{equation}
    \vw_{t+1} = \vw_t - \eta \nabla F(\vw_t).
\end{equation}
Similar to previous work \citep{wu2023implicit}, we fix the initialization $\vw_0 =
\vzero$, though our results can be extended for general initialization.

This optimization problem is convex, smooth, and Lipschitz, so classical theory
\citep{nesterov2013introductory} provides myriad guarantees for GD with sufficiently
small stepsize. In particular, since this problem is $L$-smooth with $L=1/4$, we can
show that $F(\vw_t) \leq \widetilde{\mathcal{O}}(1/(\gamma^2 t))$ when $\eta = 1/L = 4$.

However, the classical rate is not the end of the story. Recently, \citet{wu2024large}
showed that GD for $T$ iterations satisfies $F(\vw_T) \leq \mathcal{O}(1/(\gamma^4
T^2))$ if $\eta = \Theta(\gamma^2 T)$, so a large stepsize accelerates optimization for
this problem, despite the resulting non-monotonicity of the loss. In this work, we show
that for the low-dimensional case $d=2$, this rate can be further improved to
$\mathcal{O}(1/(\eta \gamma^2 T))$ for any $\eta \geq \widetilde{\Omega}(n +
1/\gamma^2)$, as long as $T \geq \Omega(n/\gamma + \log(1/\gamma)/\gamma^2)$. We achieve
this improved rate via a sharper bound of the time it takes to transition into a stable
phase, based on a fine-grained analysis of the \textit{oscillatory dynamics} of $\vw_t$
in the subspace orthogonal to $\vw_*$. We also provide a lower bound on the transition
time that matches our upper bound up to logarithmic factors.

From a higher level, we should point out that our goal here is not to achieve the
fastest optimization guarantees by any means necessary. Rather, our primary motivation
is to develop a fine-grained understanding of the unstable dynamics of GD with large
stepsizes. Given that optimization in machine learning tends to operate in unstable
regimes in practice \citep{cohen2021gradient}, we believe that it is important to
develop a rigorous mathematical understanding of unstable optimization in machine
learning with fundamental algorithms like GD, and our results here are a step in this
direction.

\paragraph{Notation} After the abstract, we use $\mathcal{O}, \Omega, \Theta$ to omit
only universal constants, and $\widetilde{\mathcal{O}}, \widetilde{\Omega}$, and
$\widetilde{\Theta}$ to omit only universal constants and factors logarithmic in $n,
1/\gamma, t, 1/\epsilon$. We denote $[n] = \{1, \ldots, n\}$.

\begin{table*}[t]
\begin{center}
\resizebox{\textwidth}{!}{
\begin{tabular}{@{}cccc@{}}
\toprule
& Complexity & Stepsize & Setting \\
\toprule
Gradient Descent \citep{nesterov2013introductory} & $\mathcal{O} \left( LB^2/\epsilon \right)$ & $\eta = 1/L$ & Convex, $L$-smooth \\
\hline
\makecell{First-Order Algorithms \\ \citep{nesterov2013introductory}} & $\Omega \left( B \sqrt{L/\epsilon} \right)$ & - & Convex, $L$-smooth \\
\hline
Gradient Descent \citep{wu2024large} & $\widetilde{\mathcal{O}} \left( \frac{1}{\gamma^2 \sqrt{\epsilon}} \right)$ & $\eta = \Theta \left( 1/\sqrt{\epsilon} \right)$ & Logistic regression \\
\hline
\makecell{Adaptive Gradient \\ Descent \citep{zhang2025gradient}} & $\mathcal{O} \left( 1/\gamma^2 \right)$ & - & Logistic regression \\
\hline
\makecell{First-Order Algorithms \\ \citep{zhang2025gradient}} & $\Omega \left( \min \left( \log n, 1/\gamma^2 \right) \right)^{(a)}$ & - & Logistic regression \\
\midrule
\rowcolor{pink} Gradient Descent (Theorem \ref{thm:upper_bound}) & $\mathcal{O} \left( \frac{n}{\gamma} + \frac{\log(1/\gamma)}{\gamma^2} \right)$ & $\eta \geq \Omega \left( \frac{1}{\epsilon (\gamma n + 1)} \right)$ & \makecell{Logistic regression \\ $d=2$} \\
\hline
\rowcolor{pink} Gradient Descent (Theorem \ref{thm:lower_bound}) & $\Omega \left( \frac{n}{\gamma} + \frac{1}{\gamma^2} \right)^{(b)}$ & $\eta \geq \widetilde{\Omega} \left( n + \frac{1}{\gamma^2} \right)$ & Logistic regression \\
\bottomrule
\end{tabular}
}
\end{center}
\caption{Iteration complexity to find an $\epsilon$-approximate solution of
\Eqref{eq:opt_prob}. Note that $(a)$ and $(b)$ show the time to find a linear separator
and time to reach $\mathcal{O}(1/\eta)$ loss, respectively, however both of these
conditions are necessary for finding an $\epsilon$-approximate solution for sufficiently
small $\epsilon$ and sufficiently large $\eta$.}
\label{tab:complexity}
\end{table*}

\subsection{Technical Overview} \label{sec:overview}
Characterizing the instability induced by large stepsizes is a fundamental difficulty of
analyzing GD in our setting. We know from previous work \citep{wu2024large} that if
$F(\vw_t) \leq 2/\eta$ for some $t$, then the loss $F(\vw_s)$ decreases monotonically
for $s \geq t$. So defining $\tau = \min \{t \geq 0 : F(\vw_t) \leq
1/8\eta\}$\footnote{\citet{wu2024large} defined $\tau$ as the first time at which
$F(\vw_t) \leq 1/\eta$, compared to our $1/8\eta$, though this difference only affects
universal constants in the analysis.}, we divide the trajectory into two phases: the
unstable phase $t \leq \tau$, where the loss may be non-monotonic, and the stable phase
$t \geq \tau$, where it is known that $F(\vw_t) \leq \widetilde{\mathcal{O}}(1/(\gamma^2
\eta t))$. The challenge is proving that GD must enter the stable phase, and bounding
the time when this happens.

\citet{wu2024large} used a ``split comparator" technique to prove that $\tau \leq
\widetilde{\mathcal{O}}(\eta/\gamma^2)$ (for sufficiently large $\eta$), which aligns
with the intuitive feeling that the length of the unstable phase might increase under
larger stepsizes. A key part of their argument is the connection between $\hat{w}_t :=
\langle \vw_t, \vw_* \rangle$ and a quantity called the gradient potential:
\begin{equation}
    G(\vw) = \frac{1}{n} \sum_{i=1}^n |\ell'(\langle \vw, \vx_i \rangle)| = \frac{1}{n} \sum_{i=1}^n \frac{1}{\exp(\langle \vw, \vx_i \rangle) + 1}.
\end{equation}
To rephrase their argument, the component $\hat{w}_t$ of $\vw_t$ in the direction of the
max-margin classifier increases at least proportionally to the gradient potential:
\begin{align} \label{eq:margin_pot}
    \hat{w}_{t+1} - \hat{w}_t = \langle \vw_{t+1} - \vw_t, \vw_* \rangle = \frac{\eta}{n} \sum_{i=1}^n \frac{\langle \vx_i, \vw_* \rangle}{\exp(\langle \vw_t, \vx_i \rangle) + 1} \geq \eta \gamma G(\vw_t),
\end{align}
where the last inequality uses that $\langle \vx_i, \vw_* \rangle \geq \gamma$. Now, to
bound $\tau$, we note that $t < \tau$ means $F(\vw_t) \geq \Omega(1/\eta)$, so that
$G(\vw_t) \geq \Omega(1/\eta)$ (by Lemma \ref{lem:pot_lb}); this means $\hat{w}_{t+1} -
\hat{w}_t \geq \Omega(\gamma)$, so $\hat{w}_t$ must increase at least \textit{linearly}
during the unstable phase. Combining with $\hat{w}_t \leq \|\vw_t\| \leq
\widetilde{\mathcal{O}}(\eta/\gamma)$ (the second inequality is proven by
\citet{wu2024large} by other means), we conclude that the unstable phase cannot last
more than $\widetilde{\mathcal{O}}(\eta/\gamma^2)$ iterations.

However, it is not known if the bound $\tau \leq \widetilde{\mathcal{O}}(\eta/\gamma^2)$
is tight, and experimental observations suggest that it is not: Figure
\ref{fig:mnist_loss_curves} shows that for MNIST data, $\tau$ does not seem to increase
with $\eta$.

\begin{figure}
\centering
\subfigure[MNIST loss and transition times.]{
\includegraphics[width=0.48\linewidth]{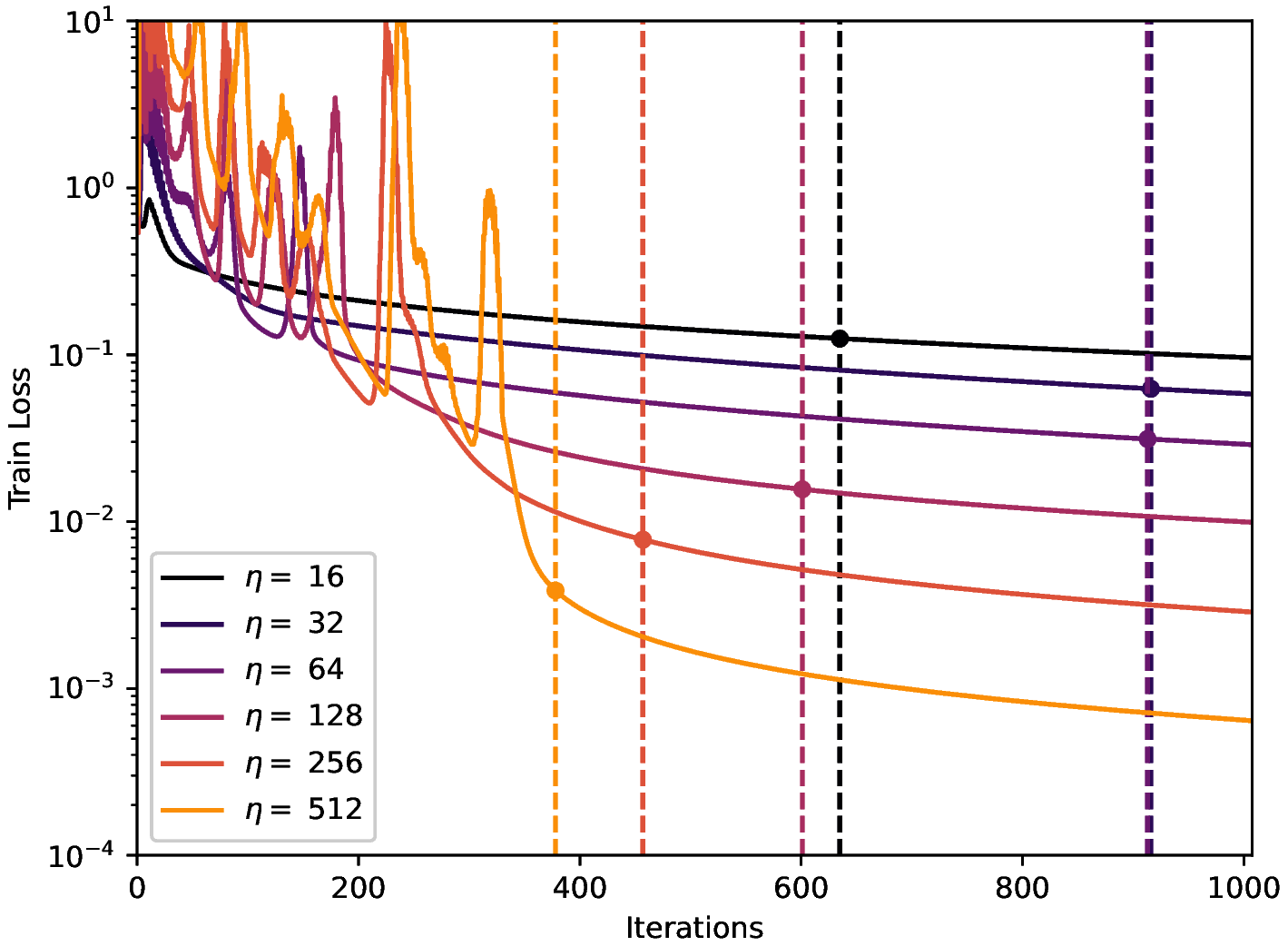}
\label{fig:mnist_loss_curves}
}
\subfigure[Example GD trajectory in two dimensions.]{
\includegraphics[width=0.48\linewidth]{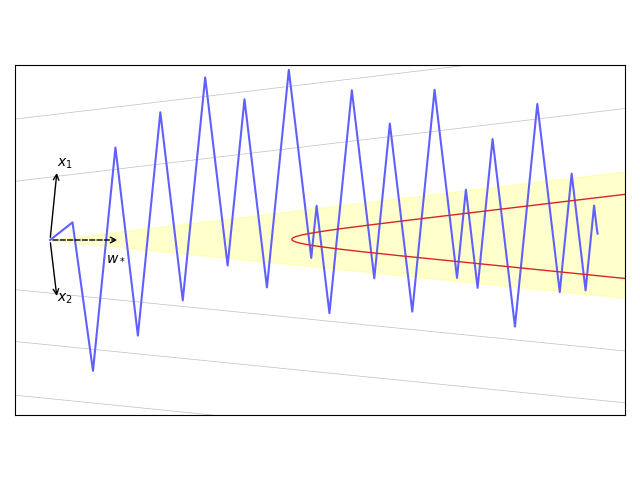}
\label{fig:example_traj}
}
\label{fig:overview}
\caption{\textbf{(a)} For GD on a subset of MNIST ($n=256$, binarized labels), larger
learning rates create instability and faster optimization. As the stepsize increases
exponentially, the stable transition time (i.e. the first timestep at which $F(\vw_t)
\leq 2/\eta$) does not increase. \textbf{(b)} Data $\vx_i$ and max-margin direction
$\vw_*$ are shown in black, GD trajectory in blue, the contour line $F(\vw) = 2/\eta$ in
red, and the region where $\langle \vw, \vx_i \rangle \geq 0$ for both $\vx_i$ in
yellow. Our proof uses the idea that avoiding the region where $F(\vw) \leq
\mathcal{O}(1/\eta)$ requires large oscillations in the subspace orthogonal to $\vw_*$,
and every such oscillation necessarily yields fast progress in the direction of
$\vw_*$.}
\end{figure}

\paragraph{Our Upper Bound}
The argument above essentially takes a \textit{static} view of the gradient potential:
we pessimistically allowed the possibility that $G(\vw_t) \approx 1/\eta$ for every $t <
\tau$. It might hold that $G(\vw_t) \approx 1/\eta$ for a single iteration, but can this
really occur at every iteration before $\tau$ along the trajectory of GD? To get a
tighter bound of $\tau$, we will take a \textit{dynamical} viewpoint by considering the
gradient potential along a trajectory that avoids the sublevel set where $F(\vw_t) \leq
1/8\eta$.

The intuition behind our proof is demonstrated in Figure \ref{fig:example_traj}. By
convexity of $F$, the gradient points into the sublevel set where $F(\vw) \leq 1/8\eta$,
so that after a certain point, the only way for GD to avoid the sublevel set is to
``jump" over it and land on the other side. Notice that the ``height" of the sublevel
set (that is, the length of each cross-section orthogonal to $\vw_*$) tends to increase
linearly along the $\vw_*$ axis, so a jump over the sublevel set requires a parameter
update with magnitude proportional to $\hat{w}_t$. More precisely, if $t$ is a step
where such a jump occurs, we show that $\hat{w}_{t+1} - \hat{w}_t \geq \gamma^2
\hat{w}_t$, that is, $\hat{w}_t$ increases \textit{exponentially}! We then show that
these oscillations across the sublevel set must happen with a certain frequency.
Compared to the linear rate of $\hat{w}_t$ from the argument of \citet{wu2024large}, our
exponential rate of $\hat{w}_t$ implies the tighter bound $\tau \leq
\widetilde{\mathcal{O}}(n/\gamma + 1/\gamma^2)$. This shows that the transition time can
be bounded independently of $\eta$, so that the stable, accelerated second phase can be
quickly reached even with arbitrarily large stepsize.

\paragraph{Our Lower Bound}
We further show that our bound $\tau \leq \widetilde{\mathcal{O}}(n/\gamma +
1/\gamma^2)$ is tight up to logarithmic factors in the worst-case (for GD) over datasets
satisfying Assumption \ref{ass:dataset}. We provide two hard instances. The first
instance requires $\Omega(n/\gamma)$ iterations until all $n$ data points are correctly
classified (which is a necessary condition for $F(\vw_t) \leq 1/8\eta$ when $\eta$ is
sufficiently large), which is a slight generalization of a lower bound construction from
\citet{tyurin2025logistic}. On the second instance, GD correctly classifies all points
quickly, but requires $\Omega(1/\gamma^2)$ iterations until $F(\vw_t) \leq 2/\eta$.

\subsection{Related Work}

\paragraph{Logistic Regression}
Many works in recent years have studied logistic regression as a fundamental testbed for
optimization in machine learning. The role played by gradient-based optimization methods
in generalization was studied through implicit bias, first for GD under separable data
\citep{soudry2018implicit} and non-separable data \citep{ji2019implicit}, and later for
SGD \citep{nacson2019stochastic} and steepest descent with momentum
\citep{gunasekar2018characterizing}.

Recently, it was shown that GD for logistic regression with separable data can converge
with any stepsize \citep{wu2023implicit}, and it was later shown that a large stepsize
could induce an accelerated convergence rate, despite the resulting instability
\citep{wu2024large}. The techniques used by \citet{wu2024large} were subsequently
applied to achieve accelerated rates in various settings, such as for two-layer networks
\citep{cai2024large}, regularized logistic regression \citep{wu2025large}, and GD with
adaptive stepsizes \citep{zhang2025gradient}. For non-separable data, the behavior of
GD with large stepsizes was explored by \citet{meng2024gradient, meng2025gradient}, who
provided negative results showing that global convergence is not guaranteed when the
stepsize is larger than a critical threshold.

\citet{zhang2025gradient} also provided lower bounds, showing that any first-order
optimization algorithm that minimizes the logistic loss requires $\Omega(\min(\log n,
1/\gamma^2))$ iterations to correctly classify all data. Also,
\citet{kornowski2024oracle} use a game-theoretic formulation to provide a lower bound of
$\Omega(1/\gamma^2)$ for finding a linear separator for algorithms that use a
``one-sided" oracle (which includes first-order optimization algorithms that minimize
the logistic loss), and another lower bound of $\Omega(1/\gamma^{2/3})$ for a broader
class of algorithms. Note that the lower bounds of \citet{kornowski2024oracle} are
formulated as the worst-case over all dataset sizes $n$, whereas our formulation (and
that of \citet{zhang2025gradient}) considers a fixed $n$.

\paragraph{Edge of Stability}
Our study is motivated by the ubiquity of unstable optimization in
practical machine learning. \citet{cohen2021gradient} discovered the Edge of Stability
(EoS) phenomenon, where GD in deep learning operates in unstable regimes, with loss not
decreasing monotonically but still tending to decrease in the long term. EoS was also
observed for adaptive optimization algorithms \citep{cohen2022adaptive}, and was later
elaborated by the central flows framework \citep{cohen2025understanding}.

Many follow up works have studied EoS theoretically. \citet{arora2022understanding}
showed that, under certain general conditions on the loss, GD at the edge of stability
follows a deterministic flow on the manifold of global minimizers. Several works have
studied surrogate models of EoS dynamics, such as 4 layer scalar networks
\citep{zhu2023understanding}, two-layer, one-neuron neural networks
\citep{chen2023beyond}, diagonal linear networks \citep{even2023s}, and a two-parameter
model of two-layer ReLU networks \citep{ahn2023learning}. \citet{damian2023self} proved
that GD at EoS has a self-stabilization property for objective functions with a
progressive sharpening property.

In this work, our goal is not to study EoS in deep learning, but rather to provide a
tight, mathematically rigorous characterization of gradient descent under unstable
regimes of a natural learning problem. See \citet{cohen2025understanding} for a
comprehensive review of the literature around EoS.

\paragraph{Large Stepsizes in Convex Optimization} GD for smooth, convex optimization
can also be accelerated by allowing large steps/non-monotonic loss. Classical methods
such as mirror descent \citep{bubeck2015convex} and Nesterov acceleration
\citep{nesterov2013introductory} do not require monotonic decrease of the loss.
\citet{malitsky2020adaptive} proposed an adaptive stepsize for gradient descent based on
local smoothness rather than global smoothness, and which does not enforce monotonic
loss decrease. \citet{altschuler2018greed} showed that the classical convergence rate of
gradient descent can be improved by constant factors, at least for a couple of
iterations, with a particular stepsize schedule that occasionally includes very big
steps. \citet{grimmer2024provably} showed that a similar improvement by constant factors
can be achieved for longer horizons (up to 127 steps). Concurrently,
\citet{altschuler2024acceleration} and \citet{grimmer2025accelerated} showed accelerated
convergence rates with stepsize schedules that include occasional large steps.
\citet{zhang2025anytime} achieved similar acceleration with an ``anytime" guarantee,
where the stepsize schedule is not defined in terms of a prior stopping time, and the
convergence guarantee holds for any stopping time. See
\citet{altschuler2024acceleration} for a thorough discussion of this line of work.

\section{Upper Bounding the Stable Transition Time} \label{sec:upper_bound}
In this section, we present our improved convergence analysis of GD for
\Eqref{eq:opt_prob} in two dimensions, which is based on a sharper analysis of the time
required for GD to transition from unstable to stable. For the entirety of this section,
we fix $d=2$.

\subsection{Statement of Results}

\begin{restatable}{theorem}{thmupperbound} \label{thm:upper_bound}
If $d=2$ and $\eta \geq \eta_0 := \max(n, 32/\gamma^2 \log(256/\gamma^2))$, then the
transition time $\tau := \min \{t \geq 0: F(\vw_t) \leq 1/8\eta\}$ of GD for
\Eqref{eq:opt_prob} satisfies
\begin{equation}
    \tau \leq \mathcal{O} \left( \frac{n}{\gamma} + \frac{\log(1/\gamma)}{\gamma^2} \right).
\end{equation}
Further, $F(\vw_t) \leq \mathcal{O} \left( 1/(\eta \gamma^2 (t - \tau)) \right)$ for all
$t > \tau$.
\end{restatable}

The key feature of the above theorem is the fact that the transition time $\tau$ is
bounded by a quantity independent of $\eta$, whereas the previously best known bound was
$\tau \leq \widetilde{\mathcal{O}}(\eta/\gamma^2)$ \citep{wu2024large}, and this
dependence is a crucial bottleneck for the overall convergence rate. Indeed, if we only
know $\tau \leq \widetilde{\mathcal{O}}(\eta/\gamma^2)$, then for a budget of $T$
iterations the largest acceptable stepsize which ensures that GD enters the stable
phase is $\eta = \widetilde{\Theta}(\gamma^2 T)$, which leads to the
$\widetilde{\mathcal{O}}(1/\gamma^4 T^2)$ rate of \citet{wu2024large}. With our improved
transition time $\tau \leq \widetilde{\mathcal{O}}(n/\gamma + 1/\gamma^2)$, GD will
definitely enter the stable phase as long as $T \geq \widetilde{\Omega}(n/\gamma +
1/\gamma^2)$, even with arbitrary large $\eta$, and such a large $\eta$ accelerates
convergence during the stable phase. This difference is shown in the complexities of
Table \ref{tab:complexity}; Theorem \ref{thm:upper_bound} implies that GD finds an
$\epsilon$-approximate solution in time independent of $\epsilon$!

Note from Table \ref{tab:complexity} that the complexity of GD with a constant stepsize
matches that of GD with the adaptive stepsize of \citet{zhang2025gradient} in terms of
the dependence on $\epsilon$ and $\gamma$, and is worse in terms of $n$. This
establishes exactly when Adaptive GD outperforms GD for this problem: if $n \leq
\mathcal{O}(1/\gamma)$, then the two algorithms have the same worst-case complexity, and
otherwise Adaptive GD is provably faster by a factor of $n \gamma$.

We can also compare GD against the lower bounds on all first-order algorithms from
\citet{zhang2025gradient}, who showed two separate lower bounds on the time to find a
linear separator: $\Omega(\min(1/\gamma^2, \log n))$ and $\Omega(\min(1/\gamma^{2/3},
n))$. In the regime of a large dataset $n \geq \Omega(\exp(1/\gamma^2))$, their combined
lower bounds simplify to $\Omega(1/\gamma^2)$, so that GD is suboptimal by a factor of
$n \gamma$, while Adaptive GD is optimal. For a small dataset $n \leq
\mathcal{O}(1/\gamma^{2/3})$, the combined lower bounds simplify to $\Omega(n)$, which
is not met by any first-order algorithm. Finally, in the intermediate regime
$\Omega(1/\gamma^{2/3}) \leq n \leq \mathcal{O}(\exp(1/\gamma^2))$, the combined lower
bounds simplify to $\Omega(1/\gamma^{2/3} + \log n)$, for which GD is suboptimal in both
$n$ and $\gamma$, and Adaptive GD is suboptimal in $\gamma$. In all regimes, for the
problem of making the loss smaller than $\epsilon$, both GD and Adaptive GD match the
optimal complexity in terms of the dependence on $\epsilon$ alone, namely both
algorithms can do so in time independent of $\epsilon$.

\subsection{Previous Bottleneck: Average-Iterate vs Last-Iterate}
In Section \ref{sec:overview}, we rephrased the argument of \citet{wu2024large} that
$\tau \leq \tilde{\mathcal{O}}(\eta/\gamma^2)$, which pessimistically allows for the
possibility that $G(\vw_t) \approx 1/\eta$ for every $t < \tau$. A related bottleneck of
their proof is the analysis of the average gradient potential $\frac{1}{t}
\sum_{s=0}^{t-1} G(\vw_s)$. Specifically, denoting $\tilde{\tau}$ as the first iteration
where $\frac{2}{t} \sum_{s=0}^{t-1} G(\vw_s) \leq 1/8\eta$, the analysis of
\citet{wu2024large} uses
\begin{equation}
    \min_{0 \leq s < t} F(\vw_s) \leq 2 \min_{0 \leq s < t} G(\vw_s) \leq \frac{2}{t} \sum_{s=0}^{t-1} G(\vw_s),
\end{equation}
(where the first inequality uses that $F(\vw) \leq 2 G(\vw)$ for all $\vw$ when $G(\vw)$
is small enough, see Lemma \ref{lem:obj_pot_lb}), and concludes that $\tau \leq
\tilde{\tau}$, then proceeds to bound $\tilde{\tau}$. However, a quick argument (which
we give below) shows that $\tilde{\tau}$ is necessarily linear in $\eta$ in the
worst-case. From our lower bound in Section \ref{sec:lower_bound}, we know that there is
a dataset for which at least one point $\vx_j$ is misclassified (i.e. $\langle \vw_t,
\vx_j \rangle \leq 0$) for the first $\Theta(n/\gamma)$ iterations. So
\begin{equation}
    G(\vw_t) = \frac{1}{n} \sum_{i=1}^n \frac{1}{\exp(\langle \vw_t, \vx_i \rangle) + 1} \geq \frac{1}{n} \frac{1}{\exp(\langle \vw_t, \vx_j \rangle) + 1} \geq \frac{1}{2n}.
\end{equation}
for the first $\Theta(n/\gamma)$ iterations. Therefore, for $t \geq \Theta(n/\gamma)$,
\begin{equation}
    \frac{2}{t} \sum_{s=0}^{t-1} G(\vw_s) \geq \frac{2}{t} \cdot \frac{1}{2n} \Theta \left( \frac{n}{\gamma} \right) = \Theta \left( \frac{1}{\gamma t} \right).
\end{equation}
Therefore, $\frac{2}{t} \sum_{s=0}^{t-1} G(\vw_s) \geq 1/8\eta$ whenever $t \leq
\mathcal{O}(\eta/\gamma)$, so $\tilde{\tau} \geq \Omega(\eta/\gamma)$. Therefore, using
$\tau \leq \tilde{\tau}$ and bounding the time-averaged potential is insufficient to
upper bound $\tau$ independently of $\eta$. To achieve such a bound, we need a more
fine-grained analysis that considers $G(\vw_t)$ for individual $t$.

\subsection{Proof of Theorem \ref{thm:upper_bound}}

\paragraph{Notation} Recall that $\vw_* = \argmax_{\|\vw\|=1} \min_{i \in [n]} \langle
\vw, \vx_i \rangle$. Choose $\vv_*$ with $\langle \vv_*, \vw_* \rangle = 0$ with
$\|\vv_*\| = 1$. We define $\hat{w}_t = \langle \vw_t, \vw_* \rangle$ and $\tilde{w}_t =
\langle \vw_t, \vv_* \rangle$. For each data $i \in [n]$, we define $\tilde{x}_i =
\langle \vx_i, \vv_* \rangle$ and $a_t^i = \langle \vw_t, \vx_i \rangle$. Note that the
loss for each data point $\ell(a_t^i)$ is decreasing in $a_t^i$, and
\begin{equation}
    a_t^i = \langle \vw_t, \vx_i \rangle = \hat{w}_t \langle \vw_*, \vx_i \rangle + \tilde{w}_t \tilde{x}_i \geq \gamma \hat{w}_t + \tilde{w}_t \tilde{x}_i,
\end{equation}
which we will use repeatedly.

Given $\eta > 0$, we define $\lambda = \log(1/(\exp(1/8\eta)-1))/\gamma$, so that
$F(\lambda \vw_*) \leq \ell(\gamma \lambda) = 1/8\eta$. We will sometimes use the
slightly larger but more convenient $\tilde{\lambda} = \log(8\eta)/\gamma$, and $\lambda
\leq \tilde{\lambda}$ can be seen by applying $\exp(1/8\eta) \geq 1 + 1/8\eta$ in the
definition of $\lambda$. Similarly, we have $F(\tilde{\lambda} \vw_*) \leq 1/8\eta$.

Denoting $\eta_0 = \max(n, 32/\gamma^2 \log(256/\gamma^2))$, we will often require $\eta
\geq \eta_0$.

At each iteration, we will split the dataset into two subsets depending on whether the
current iterate $\vw_t$ has positive or negative alignment with each $\vx_i$ in the
subspace orthogonal to $\vw_*$:
\begin{equation}
    D_t^+ = \{i \in [n] \;|\; \tilde{x}_i \tilde{w}_t \geq 0 \}, \quad D_t^- = \{i \in [n] \;|\; \tilde{x}_i \tilde{w}_t < 0 \}
\end{equation}
So for $i \in D_t^+$, we have $a_t^i \geq \gamma \hat{w}_t + \tilde{w}_t \tilde{x}_i
\geq \gamma \hat{w}_t$, which we will show is large for all $t \geq 1$. Essentially, the
loss for each data point $i \in D_t^+$ is negligible for $t \geq 1$.

~\\
We start by establishing the linear growth of $\hat{w}_t$ as discussed in Section
\ref{sec:overview} (proof in Appendix \ref{app:upper_bound}).

\begin{restatable}{lemma}{lemmaxmargin} \label{lem:max_margin}
$\hat{w}_t$ is strictly increasing. Also, if $\eta \geq \eta_0$, then $\hat{w}_t \geq
\gamma \eta/2 + \gamma (t-1)/16$ for all $1 \leq t \leq \tau$. In particular, $\hat{w}_t
\geq 8 \tilde{\lambda}$ for all $t \geq 1$.
\end{restatable}

We will say that an \textit{oscillation} occurs at iteration $t \geq 0$ if all of the
following hold: \textbf{(1)} $\hat{w}_t \geq \lambda$. \textbf{(2)} $F(\vw_t) > 1/8\eta$
and $F(\vw_{t+1}) > 1/8\eta$. \textbf{(3)} $\tilde{w}_{t+1} \tilde{w}_t < 0$, that is,
$\tilde{w}_t$ changes sign from $t$ to $t+1$.

Note that condition 2 means $t < \tau - 1$. Essentially, an oscillation happens when the
trajectory jumps over the sublevel set where $F(\vw) \leq 1/8\eta$ (see Figure
\ref{fig:example_traj}).
First, we show that $\hat{w}_t$ increases geometrically whenever an oscillation occurs,
but an oscillation can only happen when $\hat{w}_t \leq \eta/\gamma$.

\begin{lemma} \label{lem:oscillation_movement}
If an oscillation happens at iteration $t$ and $\eta \geq \eta_0$, then $\hat{w}_{t+1}
\geq (1 + \gamma^2) \hat{w}_t$.
\end{lemma}

\begin{proof}
\begin{align}
    \|\vw_{t+1} - \vw_t\| &= \eta \left\| \frac{1}{n} \sum_{i=1}^n \frac{\vx_i}{\exp(\langle \vw_t, \vx_i \rangle) + 1} \right\| \leq \frac{\eta}{n} \sum_{i=1}^n \frac{\|\vx_i\|}{\exp(\langle \vw_t, \vx_i \rangle) + 1} \\
    &\leq \frac{\eta}{n} \sum_{i=1}^n \frac{1}{\exp(\langle \vw_t, \vx_i \rangle) + 1} = \eta G(\vw_t) \leq \frac{1}{\gamma} (\hat{w}_{t+1} - \hat{w}_t),
\end{align}
where the last inequality uses $\hat{w}_{t+1} - \hat{w}_t \geq \eta \gamma
G(\vw_t)$ from \Eqref{eq:margin_pot}. Therefore
\begin{align}
    \hat{w}_{t+1} - \hat{w}_t &\geq \gamma \|\vw_{t+1} - \vw_t\| \geq \gamma |\tilde{w}_{t+1} - \tilde{w}_t| \Eqmark{i}{=} \gamma(|\tilde{w}_{t+1}| + |\tilde{w}_t|) \\
    &\Eqmark{ii}{\geq} \gamma^2 \hat{w}_{t+1}/2 + \gamma^2 \hat{w}_t/2 \Eqmark{iii}{\geq} \gamma^2 \hat{w}_t,
\end{align}
where $(i)$ uses that $\tilde{w}_t$ and $\tilde{w}_{t+1}$ have different signs
(from the definition of oscillation), $(ii)$ uses Lemma
\ref{lem:small_complement}, and $(iii)$ uses $\hat{w}_{t+1} \geq \hat{w}_t$ from
Lemma \ref{lem:max_margin}.
\end{proof}

\begin{lemma} \label{lem:oscillate_bounded_margin}
If an oscillation happens at iteration $t$ and $\eta \geq \eta_0$, then $\hat{w}_t \leq
\eta/\gamma$.
\end{lemma}

\begin{proof}
Since $t+1 < \tau$, we know from Lemma \ref{lem:small_complement} that $|\tilde{w}_t|
\geq \gamma \hat{w}_t/2$ and $|\tilde{w}_{t+1}| \geq \gamma \hat{w}_{t+1}/2 \geq \gamma
\hat{w}_t/2$, where the last inequality uses Lemma \ref{lem:max_margin}. From the
definition of oscillation, we know that $\tilde{w}_t$ and $\tilde{w}_{t+1}$ have
different signs, so $|\tilde{w}_{t+1} - \tilde{w}_t| = |\tilde{w}_{t+1}| + |\tilde{w}_t|
\geq \gamma \hat{w}_t$. Therefore
\begin{equation} 
    \gamma \hat{w}_t \leq |\tilde{w}_{t+1} - \tilde{w}_t| \leq \|\vw_{t+1} - \vw_t\| = \eta \|\nabla F(\vw_t)\| \leq \eta,
\end{equation}
where the last line uses that $F$ is $1$-Lipschitz. So $\hat{w}_t \leq \eta/\gamma$.
\end{proof}

Now we want to show that oscillations must happen at a certain frequency. To do so, we
use the following two lemmas (proofs in Appendix \ref{app:upper_bound}), which give us
rates of progress on the loss for individual data points between oscillations (Lemma
\ref{lem:a_recursion}) and a global bound for $|\tilde{w}_t|$ (Lemma
\ref{lem:complement_ub}).

\begin{restatable}{lemma}{lemarecursion} \label{lem:a_recursion}
For every $t$ with $1 \leq t < \tau$ and $i \in D_t^-$, if $\eta \geq \eta_0$ then
\begin{equation} \label{eq:a_recursion}
    a_{t+1}^i - a_t^i \geq \max \left( \frac{\eta \|\vx_i\|^2}{2n (\exp(a_t^i) + 1)}, \frac{1}{2} \eta \gamma^2 G(\vw_t) \right) 
\end{equation}
\end{restatable}

\begin{restatable}{lemma}{lemcomplementub} \label{lem:complement_ub}
If $\eta \geq \eta_0$, then $|\tilde{w}_t| \leq \eta$ for all $t < \tau$.
\end{restatable}

Lemma \ref{lem:a_recursion} tells us that between oscillations, if a data point $\vx_i$
has non-negligible loss (i.e. $i \in D_t^-$), then its loss must decrease. This
recurrence on $a_t^i$ can be combined with the initial condition $a_t^i \geq -\eta$
implied by Lemma \ref{lem:complement_ub}:
\begin{equation}
    a_t^i = \langle \vw_t, \vx_i \rangle = \hat{w}_t \langle \vw_*, \vx_i \rangle + \tilde{w}_t \langle \vv_*, \vx_i \rangle \geq -|\tilde{w}_t| \|\vv_*\| \|\vx_i\| \geq -\eta,
\end{equation}
in order to bound the time $s$ until $\ell(a_s^i) \leq 1/8\eta$ for $i \in D_t^-$. At
this point, there are two cases: either all data have low loss, implying $F(\vw_s) \leq
1/8\eta$ and we have transitioned to stability, or some points with negligible loss at
step $t$ experienced an increase in loss at step $s$, implying an oscillation has
occurred. The resulting conclusion is stated formally below and proved in Appendix
\ref{app:upper_bound}.

\begin{restatable}{lemma}{lemtimetooscillate} \label{lem:time_to_oscillate}
Suppose $\eta \geq \eta_0$ and $t < \tau$. Then for some $s \leq t + 1 + 4n/\gamma +
192/\gamma^2$, either $s = \tau$ or an oscillation happens at iteration $s$.
\end{restatable}

The last piece of the puzzle before bounding $\tau$ is to handle the steps between
oscillations; the following lemma shows that $\hat{w}_t$ will grow exponentially not
only on iterations where oscillations occur, but at any iteration before an oscillation
(proof in Appendix \ref{app:upper_bound}).

\begin{restatable}{lemma}{leminbetween} \label{lem:in_between}
For any $t$, if there exists an oscillation $t_k > t$ and $\eta \geq \eta_0$, then
$\hat{w}_{t+1} \geq (1 + \gamma^2/2) \hat{w}_t$.
\end{restatable}

Finally, we can prove our key lemma: a tight upper bound on the transition time $\tau$.

\begin{lemma} \label{lem:transition_time}
If $\eta \geq \eta_0$, then $\tau \leq 2 + 4n/\gamma + 280 \log(2/\gamma^2)/\gamma^2$.
\end{lemma}

\begin{proof}
First, we sketch the argument from a high-level. We know from Lemma \ref{lem:in_between}
that $\hat{w}_t$ grows exponentially until the last oscillation occurs. This means that
no oscillations can occur after the first $\widetilde{\mathcal{O}}(1/\gamma^2)$
iterations, since, if an oscillation were to occur after step $t =
\widetilde{\Theta}(1/\gamma^2))$ then $\hat{w}_t \geq \eta/\gamma$ (by repeated
applications of Lemma \ref{lem:in_between}) at which point Lemma
\ref{lem:oscillate_bounded_margin} implies that no further oscillations can occur. Then
by Lemma \ref{lem:time_to_oscillate}, it will be no more than
$\widetilde{\mathcal{O}}(n/\gamma + 1/\gamma^2)$ iterations until the stable transition
happens. We execute this argument below.

Let $t_0, t_1, \ldots$ be the iterations where oscillations occur. We want to show that
this list is finite, and to bound the last iteration $t_N$ where an oscillation happens.
For any oscillation $t_k$, and iteration $t < t_k$ we know from Lemma
\ref{lem:in_between} that $\hat{w}_{t+1} \geq (1 + \gamma^2/2) \hat{w}_t$, so
\begin{equation}
    \hat{w}_{t_k} \geq (1 + \gamma^2/2)^{t_k-1} \hat{w}_1 \Eqmark{i}{\geq} \frac{1}{2} (1 + \gamma^2/2)^{t_k-1} \eta \gamma, \label{eq:max_margin_oscillate_lb}
\end{equation}
where $(i)$ uses Lemma \ref{lem:max_margin}. We also know from Lemma
\ref{lem:oscillate_bounded_margin} that $\hat{w}_{t_k} \leq \eta/\gamma$, so
\begin{equation}
    \frac{1}{2} (1 + \gamma^2/2)^{t_k-1} \eta \gamma \leq \frac{\eta}{\gamma},
\end{equation}
or
\begin{equation}
    t_k \leq 1 + \frac{\log(2/\gamma^2)}{\log \left( 1 + \gamma^2/2 \right)} \Eqmark{i}{\leq} 1 + \frac{1+\gamma^2/2}{\gamma^2/2} \log(2/\gamma^2) \leq 1 + \frac{3 \log(2/\gamma^2)}{\gamma^2},
\end{equation}
where $(i)$ uses $\log(1+z) \geq z/(1+z)$ for all $z \geq 0$. Therefore, there are a
finite number $N$ of oscillations, and $t_N \leq 1 + 3 \log(2/\gamma^2)/\gamma^2$.

We can now apply Lemma \ref{lem:time_to_oscillate} with $t=t_N+1$, which implies that,
for some $s$ with $t_N < s \leq t_N + 1 + 4n/\gamma + 192/\gamma^2$, either $s$ is an
oscillation or $s = \tau$. However, $s$ cannot be an oscillation, since $t_N$ is the
last oscillation and $s > t_N$, so $s = \tau$, and
\begin{equation}
    \tau \leq t_N + 1 + \frac{4n}{\gamma} + \frac{192}{\gamma^2} \leq 2 + \frac{4n}{\gamma} + \frac{192 + 3 \log(2/\gamma^2)}{\gamma^2} \leq 2 + \frac{4n}{\gamma} + \frac{280 \log(2/\gamma^2)}{\gamma^2}.
\end{equation}
\end{proof}

To prove Theorem \ref{thm:upper_bound}, it only remains to upper bound $F(\vw_t)$ for $t
> \tau$. Rather than using the split comparator technique of \citet{wu2024large}, we
follow the analysis of \citet{crawshaw2025constant}, which explicitly bounds $\|\nabla^2
F(\vw_t)\|$ along the trajectory, then uses essentially classical descent arguments for
smooth objectives. This allows us to eliminate $\log$ terms from the final rate.

\section{Lower Bounding the Stable Transition Time} \label{sec:lower_bound}
In this section, we present a lower bound of $\tau$ that matches our upper bound up to
factors logarithmic in $1/\gamma$, implying that our analysis of GD's stable transition
time from Section \ref{sec:upper_bound} is tight.

\begin{restatable}{theorem}{thmlowerbound} \label{thm:lower_bound}
If $\gamma \leq 1/6, n \geq 2$, and $\eta \geq \eta_1 := \max \left\{ n, 32/\gamma^2
\log(3/\gamma) \right\}$, then there exists a dataset satisfying Assumption
\ref{ass:dataset} such that the transition time $\tau := \min \{t \geq 0 : F(\vw_t) \leq
1/8\eta\}$ of GD for \Eqref{eq:opt_prob} satisfies $\tau \geq \Omega \left( n/\gamma +
1/\gamma^2 \right)$.
\end{restatable}

The lower bound of \citet{zhang2025gradient} shows that any first-order optimization
algorithm for minimizing the logistic loss requires $\Omega(\min(\log n, 1/\gamma^2))$
iterations to find a linear separator. Theorem \ref{thm:lower_bound} implies that GD
requires $\Omega(n/\gamma)$ iterations for the same task\footnote{Although Theorem
\ref{thm:lower_bound} is stated in terms of finding a point with small loss, the hard
dataset in Lemma \ref{lem:classify_lb} provides a lower bound for finding a linear
separator.}, so GD is suboptimal among first-order algorithms by a factor of $n \gamma$
for large $n$.

A similar suboptimality conclusion was reached by \citet{tyurin2025logistic} for the
perceptron algorithm, although they did not consider the dependence on $\gamma$.
Specifically, the lower bound of Tyurin shows that the number of steps required by the
perceptron algorithm (or GD with $\eta \rightarrow \infty$) to find a linear separator
is $\Omega(n)$. This is worse than $\mathcal{O}(1/\gamma^2)$ required by adaptive GD to
find a separator \citep{zhang2025gradient} when $n \gg 1/\gamma^2$. We show an improved
lower bound of $\Omega((1+n\gamma)/\gamma^2)$ for large stepsize GD, implying that GD is suboptimal
when $n \gg 1/\gamma$.

Theorem \ref{thm:lower_bound} applies for sufficiently large stepsizes, and the
threshold $\eta_1$ matches that of our upper bound up to constant factors. Both hard
datasets have $d=2$, however they can be trivially generalized to any $d \geq 2$ by
embedding them into a $2$-dimensional subspace of $\mathbb{R}^d$. The conditions $n \geq
2$ and $\gamma \leq 1/6$ are to some extent unavoidable: If $n=1$, then a sufficiently
large $\eta$ will make the loss arbitrarily small in one GD step. Similarly, if $\gamma
\geq 1/\sqrt{2}$ then for every pair of data points $\vx_i, \vx_j$,
\begin{align}
    \langle \vx_i, \vx_j \rangle &= \langle \vx_i, \vw_* \rangle \langle \vx_j, \vw_* \rangle + \langle (\mI - \vw_* \vw_*^\top) \vx_i, (\mI - \vw_* \vw_*^\top) \vx_j \rangle \\
    &\geq \gamma^2 - \left\| (\mI - \vw_* \vw_*^\top) \vx_i \right\| \left\| (\mI - \vw_* \vw_*^\top) \vx_j \right\| \\
    &\geq \gamma^2 - (1 - \gamma^2) = 2 \gamma^2 - 1 \geq 0,
\end{align}
and this pairwise alignment among the dataset again implies that the loss can be made
arbitrarily small in a single GD step.

\paragraph{Proof Sketch for Theorem \ref{thm:lower_bound}} Here we informally describe
the construction of two hard datasets corresponding to the two terms of the lower bound
from Theorem \ref{thm:lower_bound}. The two datasets yield $\tau \geq \Omega(n/\gamma)$
and $\tau \geq \Omega(1/\gamma^2)$, respectively, and Theorem \ref{thm:lower_bound}
follows from $\tau \geq \Omega(\max(n/\gamma, 1/\gamma^2)) \geq \Omega(n/\gamma +
1/\gamma^2)$. The complete proof can be found in Appendix \ref{app:lower_bound}.

\begin{restatable}{lemma}{lemclassifylb}[Time until Classification] \label{lem:classify_lb}
Suppose that $\gamma \leq 1/6$, $n \geq 6$, and $\eta \geq \eta_1$. Then there exists
some dataset satisfying Assumption \ref{ass:dataset}, such that for every $t \leq n/(16
\gamma)$, there exists $i \in [n]$ with $\langle \vw_t, \vx_i \rangle < 0$.
\end{restatable}

We define the hard dataset as
\begin{align}
    \vx_i &= \begin{cases}
        (\gamma, -\gamma) & i = 1 \\
        (\gamma, \sqrt{1-\gamma^2}) & i \in \{2, \ldots, n\}
    \end{cases}.
\end{align}
Recall that the GD update can be decomposed into contributions in the direction of each
$\vx_i$:
\begin{equation}
    \vw_{t+1} = \vw_t - \frac{\eta}{n} \sum_{i=1}^n \frac{1}{\exp(\langle \vw_t, \vx_i \rangle) + 1} \vx_i.
\end{equation}
The trajectory of GD on this dataset is easy to imagine: $\vx_2, \ldots, \vx_n$ are in
agreement and have norm $1$, so together their contribution to the first update
overpowers that of $\vx_1$ (whose norm is only $\mathcal{O}(\gamma)$). After the first
step, all data points except $\vx_1$ have very low loss, so the gradient is dominated by
$\vx_1$; until $\vx_1$ is correctly classified, the GD trajectory approximately moves on
a line from $\vw_1$ in the direction of $\vx_1$. Using this, we get a recurrent upper
bound on $\langle \vw_t, \vx_i \rangle$, and we can lower bound the time until $\langle
\vw_t, \vx_i \rangle \geq 0$. Note that this construction is a slight generalization of
a lower bound for the perceptron algorithm from \citet{tyurin2025logistic}.

\begin{restatable}{lemma}{lemstablelb}[Time until Stability] \label{lem:stable_lb}
Suppose that $\gamma \leq 1/6, n \geq 2$, and $\eta \geq \eta_1$. Then there exists some
dataset satisfying Assumption \ref{ass:dataset} such that $F(\vw_t) > 2/\eta$ for all
$t \leq 1 + 1/(59 \gamma^2)$.
\end{restatable}

Denoting $k = \lceil n/2 \rceil$, we define the hard dataset as
\begin{equation}
    \vx_i = \begin{cases}
        (\gamma, -\delta) & i \leq k \\
        (\gamma, \sqrt{1-\gamma^2}) & i > k
    \end{cases},
\end{equation}
where $\delta \in [0, \sqrt{1-\gamma^2}]$. The idea is to choose $\delta$ small enough
that after the first step, the loss for $\vx_i$ is negligibly small for $i > k$, but
large enough that the loss for $\vx_i$ with $i \leq k$ is only slightly larger than
$2n/k\eta$. This would imply $F(\vw_1)$ is slightly larger than $2/\eta$; so after one
step, the GD trajectory is close to the sublevel set $F(\vw) \leq 2/\eta$, but the loss
(and the gradient norm) are small enough that it takes time to actually enter that
sublevel set. Indeed, for $i > k$, $\vx_i$ contributes negligibly to each update from
$t=1$ to $t=\tau$, so for these steps the updates to $\vw_t$ are dominated by $\vx_i$
with $i \leq k$. And, since the loss for $\vx_i$ with $i \leq k$ is quite small (only
slightly larger than $2n/k\eta$), the update size $\eta \|\nabla F(\vw_t)\|$ can be
bounded. It then requires $1/\gamma^2$ iterations until the loss for $\vx_i$ with $i
\leq k$ is smaller than $2n/k\eta$.

\section{Possible Extension to Higher Dimensions}
The most important limitation of our work is the restriction to two dimensions for the
upper bound. In this section, we discuss whether our results could be extended to higher
dimensions, and some challenges of generalizing our oscillation-based analysis for $d
\geq 2$.

\begin{figure}[t]
\centering
\subfigure[Maximum transition time as a function of $\eta$.]{
\includegraphics[width=0.48\linewidth]{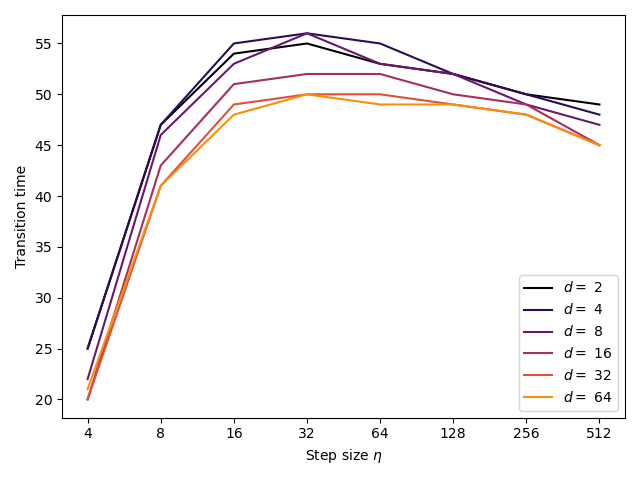}
\label{fig:tau_vs_eta}
}
\subfigure[A disconnected set of $\vw$.]{
\includegraphics[width=0.48\linewidth]{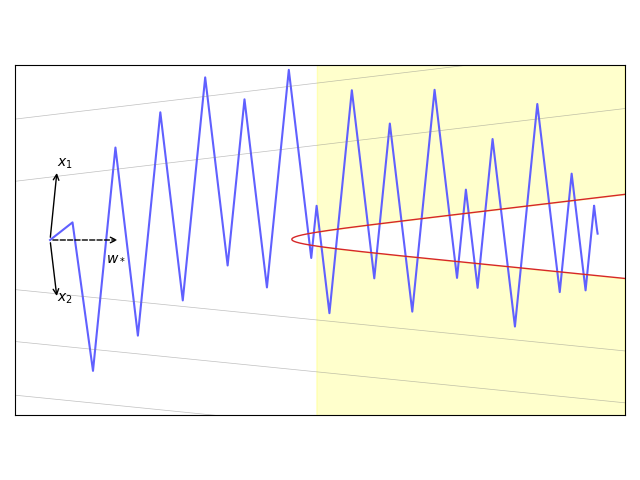}
\label{fig:disconnected}
}
\caption{\textbf{Left:} Numerical evidence suggesting that our results might hold in
higher dimensions. Over a random search of datasets, the worst-case transition time
$\tau$ does not increase past $60$ even as the learning rate $\eta$ increases
exponentially. \textbf{Right:} An illustration of a property helping our two-dimensional
analysis: the set $\{ \vw \in \mathbb{R}^d: \langle \vw, \vw_* \rangle \geq \lambda
\text{ and } \|(\mI - \vw_* \vw_*^\top) \vw\| \geq \gamma \langle \vw, \vw_* \rangle /
2\}$ is disconnected when $d=2$ (shown in yellow), which implies that the GD trajectory
must make large jumps to traverse between connected components. This set is connected
for $d > 2$.}
\label{fig:higher_d}
\end{figure}

First, we present numerical evidence suggesting that our conclusions when $d=2$ might
also hold for $d > 2$. We can see from Figure \ref{fig:tau_vs_eta} that even as $\eta$
grows exponentially, the stable transition time $\tau$ appears bounded over a random
search of many datasets with $d \geq 2$. For a given $d$, we generate a dataset of $n =
d$ samples as follows: for each $i \in [n]$ we set the max-margin component $\langle
\vx_i, \vw_* \rangle$ equal to $\gamma$, then sample the orthogonal complement
$\tilde{\vx}_i = (\mI - \vw_* \vw_*^\top) \vx_i$ uniformly over a $d-1$ dimensional ball
centered at $\vzero$ with radius $\sqrt{1-\gamma^2}$. This procedure enforces
separability with a margin of $\gamma$ and $\|\vx_i\| \leq 1$. For every $d \in \{2^1,
\ldots, 2^7\}$ and every $\eta \in \{2^2, \ldots, 2^9\}$, we generate 4096 datasets, and
compute the maximum time (over datasets) that it takes to achieve $F(\vw_t) \leq
2/\eta$. In Figure \ref{fig:tau_vs_eta}, for each $d$ we plot the worst-case $\tau$ as a
function of $\eta$. While this random search is certainly not exhaustive (especially in
very high dimensions), these results suggest that our conclusion that $\tau$ is bounded
independently of $\eta$ might extend to higher dimensions.

Indeed, some parts of our analysis (e.g. Lemma \ref{lem:max_margin}) do not rely on low
dimension. However, there are some difficulties in characterizing the oscillatory
behavior of GD's trajectory when the orthogonal component to $\vw_*$ is not a scalar,
which we discuss below. For datasets with $d \geq 2$, we denote the orthogonal component
as $\tilde{\vw}_t = (\mI - \vw_* \vw_*^\top) \vw_t$. We keep the definition $\hat{w}_t =
\langle \vw_t, \vw_* \rangle$.

First, we need a definition of oscillation that does not rely on $d$. One possibility is
to replace the condition $\tilde{w}_{t+1} \tilde{w}_t < 0$ with $\langle
\tilde{\vw}_{t+1}, \tilde{\vw}_t \rangle < 0$. With this generalization, many of our key
lemmas extend for $d \geq 2$, such as Lemma \ref{lem:oscillation_movement}, which says
$\hat{w}_t$ increases by a factor of $1 + \Omega(\gamma^2)$ when an oscillation happens,
and Lemma \ref{lem:oscillate_bounded_margin}, which says oscillations can only happen
when $\hat{w}_t \leq \eta/\gamma$.

The main obstacle is extending Lemma \ref{lem:time_to_oscillate}, that is, to show that
$\langle \tilde{\vw}_{t+1}, \tilde{\vw}_t \rangle < 0$ must happen frequently along GD
trajectories that avoid the sublevel set where $F(\vw) \leq 2/\eta$. In our
two-dimensional analysis, we showed this by implicitly leveraging a nice topological
property: for $1 \leq t < \tau$, we know that $\hat{w}_t \geq \lambda$ and
$|\tilde{w}_t| \geq \frac{1}{2} \gamma \hat{w}_t$, and the set of such $\vw$ is
\textit{disconnected} when $d=2$ (see Figure \ref{fig:disconnected}). By showing that
the trajectory cannot stay in one connected component for too long, we know that the
trajectory must ``jump" between components, which we call an oscillation. However, in
higher dimensions, the set of $\vw$ satisfying $\langle \vw, \vw_* \rangle \geq \lambda$
and $\|(\mI - \vw_* \vw_*^\top) \vw\| \geq \frac{1}{2} \gamma \langle \vw, \vw_*
\rangle$ is \textit{connected} (it is a half-space minus a cone), so it is a priori
possible for the trajectory to avoid the sublevel set without making any large jumps.
Due to this difficulty, it is unclear whether our oscillation-based analysis can
directly generalize for $d > 2$, but based on the numerical evidence in Figure
\ref{fig:tau_vs_eta}, we conjecture that $\tau \leq \widetilde{\mathcal{O}}(n/\gamma +
1/\gamma^2)$ still holds for any $d \geq 2$. We leave this problem of tightly
characterizing the transition time in general dimension open for future work.

\acks{We thank Blake Woodworth for many helpful conversations in the early stages of
this project. Michael Crawshaw is supported by the Doctoral Research Scholarship of
George Mason University. Mingrui Liu is supported by NSF grants \#2436217, \#2425687.}

\bibliography{references}

\newpage
\appendix
\section{Deferred Proofs from Section \ref{sec:upper_bound}} \label{app:upper_bound}

\lemmaxmargin*

\begin{proof}
For any $t \geq 0$,
\begin{align}
    \hat{w}_{t+1} - \hat{w}_t &= \langle \vw_{t+1} - \vw_t, \vw_* \rangle \\
    &= -\eta \left\langle \nabla F(\vw_t), \vw_* \right\rangle \\
    &= \frac{\eta}{n} \sum_{i=1}^n \frac{1}{\exp(a_t^i) + 1} \langle \vx_i, \vw_* \rangle \\
    &\geq \frac{\eta \gamma}{n} \sum_{i=1}^n \frac{1}{\exp(a_t^i) + 1} \label{eq:max_margin_inter_2} \\
    &> 0,
\end{align}
so $\hat{w}_t$ is increasing.

For $t=1$, we can bound $\hat{w}_t$ directly as
\begin{align}
    \hat{w}_1 &= \langle \vw_1, \vw_* \rangle \\
    &= \eta \langle -\nabla F(\vw_0), \vw_* \rangle \\
    &= \frac{\eta}{n} \sum_{i=1}^n \frac{\langle \vx_i, \vw_* \rangle}{\exp(a_0^i) + 1} \\
    &= \frac{\eta}{2n} \sum_{i=1}^n \langle \vx_i, \vw_* \rangle \\
    &\geq \frac{\eta \gamma}{2}.
\end{align}

For $1 < t < \tau$, we know $F(\vw_t) \geq 1/8\eta$. Therefore, by
\Eqref{eq:max_margin_inter_2} and Lemma \ref{lem:pot_lb} (together with the condition
$\eta \geq \eta_0 \geq n$),
\begin{equation}
    \hat{w}_{t+1} - \hat{w}_t \geq \eta \gamma G(\vw_t) \geq \eta \gamma \frac{1}{16 \eta} = \frac{\gamma}{16}.
\end{equation}
Unrolling back to $t=1$ yields the result.

To get $\hat{w}_t \geq 8 \tilde{\lambda}$, it suffices that $\eta \gamma/2 \geq 8
\tilde{\lambda}$, since $\hat{w}_t \geq \hat{w}_1 \geq \eta \gamma/2$. So we want
\begin{equation}
    \frac{8 \log (8\eta)}{\gamma} \leq \frac{1}{2} \eta \gamma,
\end{equation}
or equivalently
\begin{equation}
    \log (8 \eta) \leq \frac{\eta \gamma^2}{16}.
\end{equation}
By concavity of $\log$,
\begin{align}
    \log (8\eta) \leq \log \left( \frac{256}{\gamma^2} \right) + \frac{\gamma^2}{256} \left( 8\eta - \frac{256}{\gamma^2} \right) \leq \log \left( \frac{256}{\gamma^2} \right) + \frac{\eta \gamma^2}{32} \leq \frac{\eta \gamma^2}{16},
\end{align}
where the last inequality uses $\eta \geq \eta_0 \geq (32/\gamma^2) \log(128/\gamma^2)$.
\end{proof}

\begin{lemma} \label{lem:small_complement}
For $t \geq 1$, if $\eta \geq \eta_0$ and
\begin{equation}
    |\tilde{w}_t| \leq \frac{1}{2} \gamma \hat{w}_t,
\end{equation}
then $F(\vw_t) \leq 1/8\eta$.
\end{lemma}
  
\begin{proof} 
Recall that $\vw_t = \hat{w}_t \vw_* + \tilde{w}_t \vv_*$. So for each $i \in [n]$,
\begin{align} 
    a_t^i &= \langle \vw_t, \vx_i \rangle \\
    &= \hat{w}_t \langle \vw_*, \vx_i \rangle + \tilde{w}_t \langle \vv_*, \vx_i \rangle \\
    &\geq \gamma \hat{w}_t - |\tilde{w}_t| \|\vx_i\| \\
    &\Eqmark{i}{\geq} \gamma \hat{w}_t - \frac{1}{2} \gamma \hat{w}_t \\
    &= \frac{1}{2} \gamma \hat{w}_t \\
    &\Eqmark{ii}{\geq} \gamma \tilde{\lambda} \\
    &\geq \gamma \lambda,
\end{align}
where $(i)$ uses $\|\vx_i\| \leq 1$ and $|\tilde{w}_t| \leq \gamma \hat{w}_t/2$ and
$(ii)$ uses $\hat{w}_t \geq 8 \tilde{\lambda}$ from Lemma \ref{lem:max_margin}.
Therefore
\begin{align}
    F(\vw_t) = \frac{1}{n} \sum_{i=1}^n \log(1 + \exp(-a_t^i)) \leq \frac{1}{n} \sum_{i=1}^n \log(1 + \exp(-\gamma \lambda)) = 1/8\eta,
\end{align}
where the last equality follows from the definition of $\lambda$.
\end{proof}

\begin{lemma} \label{lem:active_pot}
Define $B_t = \{i \in [n]: \tilde{x}_i \tilde{w}_t \leq -\frac{1}{2} \gamma
\hat{w}_t\}$. If $\eta \geq \eta_0$, then for $t < \tau$,
\begin{equation} \label{eq:active_pot}
    \left( 1 - \frac{16}{\eta^3} \right) G(\vw_t) \leq \frac{1}{n} \sum_{i \in B_t} \frac{1}{\exp(a_t^i) + 1} \leq G(\vw_t).
\end{equation}
\end{lemma}

\begin{proof}
The second desired inequality
\begin{equation}
    \frac{1}{n} \sum_{i \in B_t} \frac{1}{\exp(a_t^i) + 1} \leq G(\vw_t)
\end{equation}
is obvious from the definition of $G(\vw_t)$, so we focus on the first desired
inequality.

For $i \notin B_t$,
\begin{equation} \label{eq:complement_inter_2}
    \langle \vw_t, \vx_i \rangle = \hat{w}_t \langle \vw_*, \vx_i \rangle + \langle \tilde{\vw}_t, \tilde{\vx}_i \rangle \geq \gamma \hat{w}_t - \frac{1}{2} \gamma \hat{w}_t = \frac{1}{2} \gamma \hat{w}_t \geq \gamma \lambda,
\end{equation}
(where the last inequality uses $\hat{w}_t \geq 8 \tilde{\lambda}$ from Lemma
\ref{lem:max_margin}), so
\begin{equation}
    \ell(\langle \vw_t, \vx_i \rangle) \leq \ell(\gamma \lambda) = 1/8\eta,
\end{equation}
and therefore
\begin{equation}
    \frac{1}{n - |B_t|} \sum_{i \notin B_t} \ell(\langle \vw_t, \vx_i \rangle) \leq 1/8\eta.
\end{equation}
Combining with the fact that $F(\vw_t) > 1/8\eta$ (since $t < \tau$), this means
$B_t$ is not empty. Essentially, $B_t$ consists of the data points that have
non-negligible loss. Now, note that for $i \notin B_t$,
\begin{align} \label{eq:active_pot_inter}
    \frac{1}{\exp(\langle \vw_t, \vx_i \rangle) + 1} &\leq \exp(-\langle \vw_t, \vx_i \rangle) \Eqmark{i}{\leq} \exp \left( -\frac{1}{2} \gamma \hat{w}_t \right) \\
    & \Eqmark{ii}{\leq} \exp(-4 \gamma \tilde{\lambda}) = 1/\eta^4 \Eqmark{iii}{\leq} 16 G(\vw_t)/\eta^3,
\end{align}
where $(i)$ uses \Eqref{eq:complement_inter_2}, $(ii)$ uses $\hat{w}_t \geq 8
\tilde{\lambda}$ from Lemma \ref{lem:max_margin}, and $(iii)$ uses $F(\vw_t) \geq
1/8\eta \implies G(\vw_t) \geq 1/(16 \eta)$ from Lemma \ref{lem:pot_lb}. Therefore
\begin{align}
    G(\vw_t) &= \frac{1}{n} \sum_{i \in B_t} \frac{1}{\exp(\langle \vw_t, \vx_i \rangle) + 1} + \frac{1}{n} \sum_{i \notin B_t} \frac{1}{\exp(\langle \vw_t, \vx_i \rangle) + 1} \\
    G(\vw_t) &\leq \frac{1}{n} \sum_{i \in B_t} \frac{1}{\exp(\langle \vw_t, \vx_i \rangle) + 1} + \frac{n - |B_t|}{n} \frac{16 G(\vw_t)}{\eta^3} \\
    G(\vw_t) &\leq \frac{1}{n} \sum_{i \in B_t} \frac{1}{\exp(\langle \vw_t, \vx_i \rangle) + 1} + \frac{16 G(\vw_t)}{\eta^3} \\
    \left( 1 - \frac{16}{\eta^3} \right) G(\vw_t) &\leq \frac{1}{n} \sum_{i \in B_t} \frac{1}{\exp(\langle \vw_t, \vx_i \rangle) + 1}
\end{align}
which proves the first inequality in \Eqref{eq:active_pot}.
\end{proof}

\lemarecursion*

\begin{proof}
The idea is that the data points in $D_t^+$ all have low loss, so they don't contribute
much to each gradient update, while the contributions of each point in $D_t^-$ are all
``aligned" in that each pair has positive dot product.

To make this concrete, denote $m = \argmin_{i \in [n]} a_t^i$, and notice that $m \in
D_t^-$, since for every $i \in D_t^+$:
\begin{equation}
    a_t^i = \langle \vw_t, \vx_i \rangle = \hat{w}_t \langle \vw_*, \vx_i \rangle + \tilde{w}_t \langle \vv_*, \vx_i \rangle \Eqmark{i}{\geq} \hat{w}_t \langle \vw_*, \vx_i \rangle \geq \gamma \hat{w}_t \geq \gamma \tilde{\lambda},
\end{equation}
(where $(i)$ uses $\tilde{w}_t \langle \vv_*, \vx_i \rangle \geq 0$ from the definition
of $D_t^+$) and we know that $\min_{i \in [n]} a_t^i < \gamma \tilde{\lambda}$, since
otherwise $F(\vw_t) \leq \ell(\gamma \tilde{\lambda}) \leq 1/8\eta$. For each $i \in
D_t^-$, we can now bound the change in $a_t^i$ as
follows:
\begin{align}
    a_{t+1}^i - a_t^i &= \langle \vw_{t+1} - \vw_t, \vx_i \rangle \\
    &= \frac{\eta}{n} \sum_{j=1}^n \frac{\langle \vx_j, \vx_i \rangle}{\exp(a_t^j) + 1} \\
    &= \eta \Bigg( \underbrace{\frac{1}{n} \sum_{j \in D_t^-} \frac{\langle \vx_j, \vx_i \rangle}{\exp(a_t^j) + 1}}_{A_1} + \underbrace{\frac{1}{n} \sum_{j \in D_t^+} \frac{\langle \vx_j, \vx_i \rangle}{\exp(a_t^j) + 1}}_{A_2} \Bigg).
\end{align}
We will show that $A_1$ dominates $A_2$. Note that
\begin{align}
    A_1 &\Eqmark{i}{\geq} \frac{\gamma^2}{n} \sum_{j \in D_t^-} \frac{1}{\exp(a_t^j) + 1} \Eqmark{ii}{\geq} \frac{\gamma^2}{n} \sum_{j \in B_t} \frac{1}{\exp(a_t^j) + 1} \Eqmark{iii}{\geq} \gamma^2 \left( 1 - \frac{16}{\eta^3} \right) G(\vw_t) \Eqmark{iv}{\geq} \frac{9}{10} \gamma^2 G(\vw_t),
\end{align}
where $(i)$ uses that for $i, j \in D_t^-$,
\begin{align} \label{eq:common_dot}
    \langle \vx_i, \vx_j \rangle = \langle \vx_i, \vw_* \rangle \langle \vx_j, \vw_* \rangle + \langle \vx_i, \vv_* \rangle \langle \vx_j, \vv_* \rangle \geq \langle \vx_i, \vw_* \rangle \langle \vx_j, \vw_* \rangle \geq \gamma^2,
\end{align}
$(ii)$ uses that $B_t \subset D_t^-$, $(iii)$ uses Lemma \ref{lem:active_pot}, and
$(iv)$ uses $\eta \geq \eta_0 \geq 32 \log(256)$. For $A_2$,
\begin{align}
    |A_2| \leq \frac{1}{n} \sum_{j \in D_t^+} \frac{1}{\exp(a_t^j) + 1} \Eqmark{i}{\leq} \frac{1}{n} \sum_{j \notin B_t} \frac{1}{\exp(a_t^j) + 1} \Eqmark{ii}{\leq} \frac{16}{\eta^3} G(\vw_t),
\end{align}
where $(i)$ uses $B_t \subset D_t^- \implies D_t^+ \subset [n] \backslash B_t$ and
$(ii)$ uses Lemma \ref{lem:active_pot}. Since $\eta \geq \eta_0 \geq 32
\log(256)/\gamma^2 \geq 128/\gamma^2$, this means
\begin{align}
    |A_2| \leq \frac{1}{2^{17}} \gamma^2 G(\vw_t) \leq \frac{1}{2^{16}} A_1,
\end{align}
so
\begin{align}
    a_{t+1}^i - a_t^i = \eta (A_1 + A_2) \geq \frac{2}{3} \eta A_1 = \frac{2 \eta}{3n} \sum_{j \in D_t^-} \frac{\langle \vx_j, \vx_i \rangle}{\exp(a_t^j) + 1}.
\end{align}
Now we can derive the two desired bounds. First, recall that $i \in D_t^-$, so
\begin{align}
    a_{t+1}^i - a_t^i &\geq \frac{2 \eta}{3n} \frac{\|\vx_i\|^2}{\exp(a_t^i) + 1} \geq \frac{\eta \|\vx_i\|^2}{2n (\exp(a_t^i) + 1)}.
\end{align}
Second,
\begin{align}
    a_{t+1}^i - a_t^i &\geq \frac{2 \eta}{3n} \sum_{j \in B_t} \frac{\langle \vx_j, \vx_i \rangle}{\exp(a_t^j) + 1} \Eqmark{i}{\geq} \frac{2 \eta \gamma^2}{3n} \sum_{j \in B_t} \frac{1}{\exp(a_t^j) + 1} \\
    &\Eqmark{ii}{\geq} \frac{2}{3} \left( 1 - \frac{16}{\eta^3} \right) \eta \gamma^2 G(\vw_t) \Eqmark{iii}{\geq} \frac{1}{2} \eta \gamma^2 G(\vw_t).
\end{align}
where $(i)$ uses $\langle \vx_j, \vx_i \rangle \geq \gamma^2$ as in
\Eqref{eq:common_dot}, $(ii)$ uses Lemma \ref{lem:active_pot}, and $(iii)$ uses $\eta
\geq \eta_0 \geq 32 \log(256)$.
\end{proof}

\lemcomplementub*

\begin{proof}
Since $\vw_0 = \vzero$, it clearly holds for $t=0$, and it also holds for $t=1$ since
\begin{equation}
    |\tilde{w}_1| \leq \|\vw_1\| = \|\vw_0 - \eta \nabla F(\vw_0)\| = \eta \|\nabla F(\vzero)\| \leq \eta,
\end{equation}
using that $F$ is $1$-Lipschitz.

Now suppose $|\tilde{w}_t| \leq \eta$ for some $t \geq 1$. We consider two cases. If an
oscillation happens at step $t$, then
\begin{equation}
    |\tilde{w}_{t+1}| \leq |\tilde{w}_{t+1}| + |\tilde{w}_t| \Eqmark{i}{=} |\tilde{w}_{t+1} - \tilde{w}_t| = \eta |\langle \nabla F(\vw_t), \vv_* \rangle| \leq \eta \|\nabla F(\vw_t)\| \Eqmark{ii}{\leq} \eta,
\end{equation}
where $(i)$ uses that $\tilde{w}_{t+1}$ and $\tilde{w}_t$ have opposite sign, and $(ii)$
again uses that $F$ is $1$-Lipschitz.

Now suppose an oscillation does not happen at step $t$, so that $\tilde{w}_{t+1}$ and
$\tilde{w}_t$ have the same sign. First, we claim that $\tilde{w}_{t+1} - \tilde{w}_t$
has opposite sign from $\tilde{w}_t$. To see why, recall the definition of $B_t = \{i
\in [n]: \tilde{x}_i \tilde{w}_t \leq -\frac{1}{2} \gamma \hat{w}_t\}$ from Lemma
\ref{lem:active_pot}, and notice
\begin{align}
    \tilde{w}_t (\tilde{w}_{t+1} - \tilde{w}_t) &= \frac{\eta}{n} \sum_{i=1}^n \frac{\tilde{w}_t \tilde{x}_i}{\exp(a_t^i) + 1} \\
    &= \eta \bigg( \underbrace{\frac{1}{n} \sum_{i \in B_t} \frac{\tilde{w}_t \tilde{x}_i}{\exp(a_t^i) + 1}}_{A_1} + \underbrace{\frac{1}{n} \sum_{i \notin B_t} \frac{\tilde{w}_t \tilde{x}_i}{\exp(a_t^i) + 1}}_{A_2} \bigg).
\end{align}
We can bound $A_1$ as
\begin{align}
    A_1 &= \frac{1}{n} \sum_{i \in B_t} \frac{\tilde{w}_t \tilde{x}_i}{\exp(a_t^i) + 1} \leq -\frac{\gamma \hat{w}_t}{2n} \sum_{i \in B_t} \frac{1}{\exp(a_t^i) + 1} \Eqmark{i}{\leq} -\frac{1}{2} \gamma \hat{w}_t \left( 1 - \frac{16}{\eta^3} \right) G(\vw_t) \\
    &\Eqmark{ii}{\leq} -\frac{1}{4} \eta \gamma^2 \left( 1 - \frac{16}{\eta^3} \right) G(\vw_t) \Eqmark{iii}{\leq} -38 G(\vw_t),
\end{align}
where $(i)$ uses Lemma \ref{lem:active_pot}, $(ii)$ uses $\hat{w}_t \geq \hat{w}_1 \geq
\eta \gamma/2$ from Lemma \ref{lem:max_margin}, and $(iii)$ uses $\eta \geq \eta_0 \geq
32 \log(256)/\gamma^2$. Similarly, we can bound $A_2$ as
\begin{align}
    A_2 &= \frac{1}{n} \sum_{i \notin B_t} \frac{\tilde{w}_t \tilde{x}_i}{\exp(a_t^i) + 1} \Eqmark{i}{\leq} \frac{\eta}{n} \sum_{i \notin B_t} \frac{1}{\exp(a_t^i) + 1} \Eqmark{ii}{\leq} \frac{16}{\eta^2} G(\vw_t) \Eqmark{iii}{\leq} G(\vw_t),
\end{align}
where $(i)$ uses $\tilde{w}_t \tilde{x}_i \leq |\tilde{w}_t| \leq \eta$, $(ii)$ uses
Lemma \ref{lem:active_pot}, and $(iii)$ uses $\eta \geq \eta_0 \geq 16$. So
\begin{equation}
    \tilde{w}_t (\tilde{w}_{t+1} - \tilde{w}_t) = \eta(A_1 + A_2) < 0,
\end{equation}
which proves the claim. So $\text{sign}(\tilde{w}_t) = \text{sign}(\tilde{w}_{t+1})
= -\text{sign}(\tilde{w}_{t+1} - \tilde{w}_t)$, and
\begin{align}
    |\tilde{w}_{t+1} - (\tilde{w}_{t+1} - \tilde{w}_t)| &= |\tilde{w}_{t+1}| + |\tilde{w}_{t+1} - \tilde{w}_t| \\
    |\tilde{w}_t| &= |\tilde{w}_{t+1}| + |\tilde{w}_{t+1} - \tilde{w}_t| \\
    |\tilde{w}_{t+1}| &= |\tilde{w}_t| - |\tilde{w}_{t+1} - \tilde{w}_t| \\
    |\tilde{w}_{t+1}| &\leq |\tilde{w}_t| \leq \eta.
\end{align}
\end{proof}

\lemtimetooscillate*

\begin{proof}
The idea is to analyze $a_t^i$ in two phases: first we show that $a_s^i \geq 0$ for all
$i$ after $\mathcal{O}(n/\gamma)$ iterations, then we derive a recurrence relation on
$G(\vw_s)$ showing that $G(\vw_s) \leq 1/\eta$ after $1/\gamma^2$ iterations, unless an
oscillation happens first.

Let $t_{\text{osc}}$ be the first iteration after $t$ such that either an oscillation
happens at step $t_{\text{osc}}$ or $\tau = t_{\text{osc}}$. From the definition of an
oscillation, we know that $\langle \vw_s, \vv_* \rangle$ does not change sign for $s \in
\{t, \ldots, t_{\text{osc}}\}$, so $D_t^+ = D_{t+1}^+ = \ldots = D_{t_{\text{osc}}}^+$
and $D_t^- = D_{t+1}^- = \ldots = D_{t_{\text{osc}}}^-$. Accordingly, we denote $D_+ =
D_t^+$ and $D_- = D_t^-$.

First, we bound the number of steps until $a_s^i \geq 0$ for every $i$. Note that $a_s^i
\geq 0$ for every $i \in D_+$, since
\begin{equation} \label{eq:oscillate_inter}
    a_s^i = \langle \vw_s, \vx_i \rangle = \hat{w}_s \langle \vw_*, \vx_i \rangle + \tilde{w}_s \langle \vv_*, \vx_i \rangle \Eqmark{i}{\geq} \gamma \hat{w}_s \Eqmark{ii}{\geq} 0,
\end{equation}
where $(i)$ uses $\tilde{w}_s \langle \vv_*, \vx_i \rangle \geq 0$ from the definition
of $D_s^+$ and $(ii)$ uses $\hat{w}_s \geq \hat{w}_0 = 0$ from Lemma
\ref{lem:max_margin}. So we only need to worry about $a_s^i$ for $i \in D_-$.

Suppose for some $i \in D_-$ that $a_t^i < 0$, and let $s \geq t$ such that $a_r^i < 0$
for all $r$ with $t \leq r \leq s$. Then according to Lemma \ref{lem:a_recursion},
\begin{align}
    a_{s+1}^i - a_s^i &\geq \frac{\eta \|\vx_i\|^2}{2n (\exp(a_s^i) + 1)} \geq \frac{\eta \|\vx_i\|^2}{4n},
\end{align}
so that
\begin{align}
    a_s^i - a_t^i &\geq \frac{\eta \|\vx_i\|^2 (s-t)}{4n}.
\end{align}
This means that $a_s^i \geq 0$ when
\begin{align}
    s - t \geq \frac{4n \max(0, -a_t^i)}{\eta \|\vx_i\|^2}.
\end{align}
Note that
\begin{align}
    a_t^i = \langle \vw_t, \vx_i \rangle = \hat{w}_t \langle \vw_*, \vx_i \rangle + \tilde{w}_t \langle \vv_*, \vx_i \rangle \geq -|\tilde{w}_t| \|\vv_*\| \|\vx_i\| \geq -\eta \|\vx_i\|,
\end{align}
where the last inequality uses $|\tilde{w}_t| \leq \eta$ from Lemma
\ref{lem:complement_ub} and $\|\vv_*\| = 1$.  Therefore $a_s^i \geq 0$ when $s -
t \geq 4n/\|\vx_i\|$. Since $\|\vx_i\| \geq \langle \vx_i, \vw_* \rangle \geq
\gamma$, it suffices that $s - t \geq 4n/\gamma$.

Let $t_0$ be the smallest $s \geq t$ such that $a_s^i \geq 0$ for all $i$, so that by
the above $t_0 \leq t + 4n/\gamma$. Note also that $a_s^i \geq 0$ for all $s \in \{t_0,
\ldots, t_{\text{osc}} - 1\}$, since \Eqref{eq:oscillate_inter} holds for all such $s,
i$ and $a_s^i$ is increasing for $i \in D_-$ by Lemma \ref{lem:a_recursion}.

Now we consider $G(\vw_s)$ for the second phase. For all $s \geq t_0$,
\begin{align}
    G(\vw_{s+1}) &= \frac{1}{n} \sum_{i=1}^n \frac{1}{\exp(a_{s+1}^i) + 1} \label{eq:phase2_inter_1} \\
    &\Eqmark{i}{\leq} \frac{\eta^3}{\eta^3 - 16} \frac{1}{n} \sum_{i \in D_-} \frac{1}{\exp(a_{s+1}^i) + 1} \\
    &\leq \frac{\eta^3}{\eta^3 - 16} \frac{1}{n} \sum_{i \in D_-} \exp(-a_{s+1}^i) \\
    &\Eqmark{ii}{\leq} \exp \left( -\frac{1}{2} \eta \gamma^2 G(\vw_s) \right) \frac{\eta^3}{\eta^3 - 16} \frac{1}{n} \sum_{i \in D_-} \exp(-a_s^i) \\
    &\Eqmark{iii}{\leq} \exp \left( -\frac{1}{2} \eta \gamma^2 G(\vw_s) \right) \frac{2 \eta^3}{\eta^3 - 16} \frac{1}{n} \sum_{i \in D_-} \frac{1}{\exp(a_s^i) + 1} \\
    &\Eqmark{iv}{\leq} \frac{129}{64} \exp \left( -\frac{1}{2} \eta \gamma^2 G(\vw_s) \right) \frac{1}{n} \sum_{i \in D_-} \frac{1}{\exp(a_s^i) + 1} \\
    &\leq \frac{129}{64} \exp \left( -\frac{1}{2} \eta \gamma^2 G(\vw_s) \right) G(\vw_s), \label{eq:phase2_inter_2}
\end{align}
where $(i)$ uses Lemma \ref{lem:active_pot} together with $B_{s+1} \subset D_s^- = D_-$, $(ii)$
uses Lemma \ref{lem:a_recursion}, $(iii)$ uses $a_s^i \geq 0$, and $(iv)$ uses $\eta
\geq \eta_0 \geq 32 \log(256)$. We can first bound $G(\vw_{t_0+1})$ as follows. Denoting
$b = \frac{1}{2} \eta \gamma^2$ and $T(z) = 129/64 \exp(-bz) z$, we have $G(\vw_{s+1})
\leq T(G(\vw_s))$. $T$ has a global maximum at $z = 1/b$, so that
\begin{equation} \label{eq:phase2_first_step}
    G(\vw_{t_0+1}) \leq T(1/b) = \frac{129}{64} \exp(-1)/b = \frac{129}{32e \eta \gamma^2} \leq \frac{3}{2 \eta \gamma^2}.
\end{equation}
Knowing this, we perform a similar, slightly stronger derivation as in Equations
\ref{eq:phase2_inter_1} through \ref{eq:phase2_inter_2}. For $s \geq t_0 + 1$,
\begin{align}
    G(\vw_{s+1}) &= \frac{1}{n} \sum_{i=1}^n \frac{1}{\exp(a_{s+1}^i) + 1} \\
    &\Eqmark{i}{\leq} \left( 1 + \frac{16}{\eta^3 - 16} \right) \frac{1}{n} \sum_{i \in B_{s+1}} \frac{1}{\exp(a_{s+1}^i) + 1} \\
    &\Eqmark{ii}{\leq} \left( 1 + \frac{16}{\eta^3 - 16} \right) \frac{1}{n} \sum_{i \in B_{s+1}} \frac{1}{\exp(a_s^i) \exp(\eta \gamma^2 G(\vw_s)/2) + 1} \\
    &\Eqmark{iii}{\leq} \left( 1 + \frac{16}{\eta^3 - 16} \right) \exp \left( -\frac{1}{8} \eta \gamma^2 G(\vw_s) \right) \frac{1}{n} \sum_{i \in B_{s+1}} \frac{1}{\exp(a_s^i) + 1} \\
    &\leq \left( 1 + \frac{16}{\eta^3 - 16} \right) \exp \left( -\frac{1}{8} \eta \gamma^2 G(\vw_s) \right) G(\vw_s), \label{eq:phase2_inter}
\end{align}
where $(i)$ uses Lemma \ref{lem:active_pot}, $(ii)$ uses Lemma \ref{lem:a_recursion}
together with $B_{s+1} \subset D_s^- = D_-$, and $(iii)$ can be justified as follows:
denoting $C = \exp(\eta \gamma^2 G(\vw_s)/2)$, we want to show that
\begin{equation}
    C \exp(a_s^i) + 1 \geq C^{1/4} (\exp(a_s^i) + 1),
\end{equation}
or equivalently,
\begin{align}
    C^{1/4} &\leq \frac{C \exp(a_s^i) + 1}{\exp(a_s^i) + 1} \\
    (\iff) \quad \frac{1}{4} \log C &\leq \log \left( \frac{C \exp(a_s^i) + 1}{\exp(a_s^i) + 1} \right) \\
    (\iff) \quad \frac{1}{4} &\leq \log \left( 1 + \frac{(C-1) \exp(a_s^i)}{\exp(a_s^i) + 1} \right) / \log C \\
    (\impliedby) \quad \frac{1}{4} &\Eqmark{iv}{\leq} \log \left( 1 + \frac{(C-1) \exp(a_s^i)}{\exp(a_s^i) + 1} \right) / (C - 1) \\
    (\impliedby) \quad \frac{1}{4} &\Eqmark{v}{\leq} \frac{(C-1) \exp(a_s^i)}{\exp(a_s^i) + 1} \frac{\exp(a_s^i) + 1}{C \exp(a_s^i) + 1} \frac{1}{C-1} \\
    (\iff) \quad \frac{1}{4} &\leq \frac{1}{C + \exp(-a_s^i)} \\
    (\iff) \quad 4 &\geq C + \exp(-a_s^i),
\end{align}
where $(iv)$ uses $\log(1 + z) \leq z$ and $(v)$ uses $\log(1 + z) \geq z/(1+z)$. The
final condition is satisfied since $a_s^i \geq 0$ and $G(\vw_s) \leq 3/(2 \eta
\gamma^2)$ from \Eqref{eq:phase2_first_step} implies $C \leq \exp(3/4) \leq 3$. This
justifies $(iii)$.

From \Eqref{eq:phase2_inter},
\begin{align}
    G(\vw_{s+1}) &\leq \left( 1 + \frac{16}{\eta^3 - 16} \right) \exp \left( -\frac{1}{8} \eta \gamma^2 G(\vw_s) \right) G(\vw_s) \\
    &\Eqmark{i}{\leq} \left( 1 + \frac{16}{\eta^3 - 16} \right) \left( 1 - \frac{\eta \gamma^2 G(\vw_s)/8}{1 + \eta \gamma^2 G(\vw_s)/8} \right) G(\vw_s) \\
    &\Eqmark{ii}{\leq} \left( 1 + \frac{16}{\eta^3 - 16} \right) \left( 1 - \frac{1}{10} \eta \gamma^2 G(\vw_s) \right) G(\vw_s) \\
    &\leq \left( 1 - \frac{1}{10} \eta \gamma^2 G(\vw_s) + \frac{16}{\eta^3 - 16} \right) G(\vw_s) \\
    &\Eqmark{iii}\leq \left( 1 - \frac{1}{10} \eta \gamma^2 G(\vw_s) + \frac{32}{\eta^3} \right) G(\vw_s) \\
    &= \left( 1 - \left( \frac{1}{10} - \frac{32}{\eta^4 \gamma^2 G(\vw_s)} \right) \eta \gamma^2 G(\vw_s) \right) G(\vw_s) \\
    &\Eqmark{iv}{\leq} \left( 1 - \left( \frac{1}{10} - \frac{512}{\eta^3 \gamma^2} \right) \eta \gamma^2 G(\vw_s) \right) G(\vw_s) \\
    &\Eqmark{v}{\leq} \left( 1 - \frac{1}{12} \eta \gamma^2 G(\vw_s) \right) G(\vw_s)
\end{align}
where $(i)$ uses $\exp(z) \leq 1/(1-z) = 1 + z/(1-z)$, $(ii)$ uses $G(\vw_s)
\leq 3/(2 \eta \gamma^2)$ from \Eqref{eq:phase2_first_step}, $(iii)$ uses $\eta
\geq \eta_0 \geq 32$, $(iv)$ uses $F(\vw_s) > 1/(8 \eta) \implies G(\vw_s) \geq
1/(16\eta)$ by Lemma \ref{lem:pot_lb}, and $(v)$ uses $\eta \geq \eta_0 \geq 32
\log(128)/\gamma^2$.  Finally, we can unroll this recurrence over $s$ after some
manipulation:
\begin{align}
    \frac{1}{G(\vw_s)} &\leq \frac{1}{G(\vw_{s+1})} - \frac{1}{12} \eta \gamma^2 \frac{G(\vw_s)}{G(\vw_{s+1})} \\
    \frac{1}{G(\vw_s)} &\leq \frac{1}{G(\vw_{s+1})} - \frac{1}{12} \eta \gamma^2 \\
    \frac{1}{G(\vw_{s+1})} &\geq \frac{1}{G(\vw_s)} + \frac{1}{12} \eta \gamma^2 \\
    \frac{1}{G(\vw_s)} &\geq \frac{1}{G(\vw_{t_0+1})} + \frac{1}{12} \eta \gamma^2 (s - (t_0 + 1)) \\
    \frac{1}{G(\vw_s)} &\geq \frac{1}{12} \eta \gamma^2 (s - (t_0 + 1)) \\
    G(\vw_s) &\leq \frac{12}{\eta \gamma^2 (s - (t_0 + 1))}.
\end{align}
Now define $t_1 = 1 + t_0 + 192/\gamma^2$. If $\tau \geq t_1$ and no oscillation happens
between $t$ and $t_1$, then $G(\vw_{t_1}) \leq 1/(16 \eta)$, which implies $F(\vw_{t_1})
\leq 1/(8\eta)$ by Lemma \ref{lem:obj_pot_lb}, so that $\tau = t_1$.
\end{proof}

\leminbetween*

\begin{proof}
Without loss of generality, assume that $t_k$ is the earliest iteration where an
oscillation occurs after $t$.

Similarly to the proof of Lemma \ref{lem:time_to_oscillate}, we can partition the
dataset into $D_+$ and $D_-$, where $D_+ = D_s^+$ and $D_- = D_s^-$ for all $s \in \{t,
\ldots, t_k-1\}$. We can then apply Lemma \ref{lem:a_recursion} for each iteration from
$t$ to $t_k$ and conclude that $a_t^i \leq a_{t_k}^i$ for $i \in D_-$. Therefore
\begin{align}
    \hat{w}_{t+1} - \hat{w}_t &= \eta \left\langle \vw_*, -\nabla F(\vw_t) \right\rangle \\
    &= \frac{\eta}{n} \sum_{i=1}^n \frac{\langle \vw_*, \vx_i \rangle}{\exp(a_t^i) + 1} \\
    &\geq \frac{\eta}{n} \sum_{i \in D_-} \frac{\langle \vw_*, \vx_i \rangle}{\exp(a_t^i) + 1} \\
    &\geq \frac{\eta}{n} \sum_{i \in D_-} \frac{\langle \vw_*, \vx_i \rangle}{\exp(a_{t_k}^i) + 1} \\
    &= \eta \Bigg( \underbrace{\frac{1}{n} \sum_{i=1}^n \frac{\langle \vw_*, \vx_i \rangle}{\exp(a_{t_k}^i) + 1}}_{A_1} - \underbrace{\frac{1}{n} \sum_{i \in D_+} \frac{\langle \vw_*, \vx_i \rangle}{\exp(a_{t_k}^i) + 1}}_{A_2} \Bigg). \label{eq:inbetween_inter}
\end{align}
Note that
\begin{equation}
    A_1 \geq \frac{\gamma}{n} \sum_{i=1}^n \frac{1}{\exp(a_{t_k}^i)+1} = \gamma G(\vw_{t_k}),
\end{equation}
and
\begin{equation}
    |A_2| \leq \frac{1}{n} \sum_{i \in D_+} \frac{1}{\exp(a_{t_k}^i) + 1} \Eqmark{i}{\leq} \frac{1}{n} \sum_{i \notin B_{t_k}} \frac{1}{\exp(a_{t_k}^i) + 1} \Eqmark{ii}{\leq} \frac{16}{\eta^3} G(\vw_{t_k}),
\end{equation}
where $(i)$ uses $B_{t_k} \subset D_-$ and $(ii)$ uses Lemma \ref{lem:active_pot}. Using
$\eta \geq \eta_0 \geq 32/\gamma^2$, this means
\begin{equation}
    |A_2| \leq \frac{16}{\eta^3} G(\vw_{t_k}) \leq \frac{1}{2 \eta^2} \gamma^2 G(\vw_{t_k}) \leq \frac{1}{2} \gamma G(\vw_{t_k}) \leq \frac{1}{2} A_1.
\end{equation}
Plugging back to \Eqref{eq:inbetween_inter} yields
\begin{align}
    \hat{w}_{t+1} - \hat{w}_t &\geq \frac{1}{2} \eta A_1 \\
    &\geq \frac{\eta}{2n} \sum_{i=1}^n \frac{\langle \vw_*, \vx_i \rangle}{\exp(a_{t_k}^i) + 1} \\
    &= \frac{\eta}{2} \langle \vw_*, -\nabla F(\vw_{t_k}) \rangle \\
    &= \frac{1}{2} (\hat{w}_{t_k+1} - \hat{w}_{t_k}) \\
    &\Eqmark{i}{\geq} \frac{1}{2} \gamma^2 \hat{w}_{t_k} \\
    &\Eqmark{ii}{\geq} \frac{1}{2} \gamma^2 \hat{w}_t,
\end{align}
where $(i)$ uses Lemma \ref{lem:oscillation_movement} on iteration $t_k$ and $(ii)$ uses
that $\hat{w}_t$ is strictly increasing (Lemma \ref{lem:max_margin}).
\end{proof}

\thmupperbound*

\begin{proof}
The upper bound of $\tau$ was shown in Lemma \ref{lem:transition_time}. For the upper
bound of $F(\vw_t)$, we first show that $F(\vw_t) \leq 1/8\eta$ for all $t \geq \tau$.
This is essentially a repetition of parts of the analysis from \citet{wu2024large} and
\citet{crawshaw2025constant}; we include it here for completeness.

We already know that $F(\vw_{\tau}) \leq 1/8\eta$, so assume that $F(\vw_t) \leq
1/8\eta$ for $t \geq \tau$. We can show that $F(\vw_{t+1}) \leq F(\vw_t)$ by applying
the modified descent inequality of Lemma \ref{lem:stable_descent}, but to do so we need
to verify the condition $\|\vw_{t+1} - \vw_t\| \leq 1$:
\begin{equation}
    \|\vw_{t+1} - \vw_t\| = \eta \|\nabla F(\vw_t)\| \Eqmark{i}{\leq} \eta F(\vw_t) \Eqmark{ii}{\leq} 1/8,
\end{equation}
where $(i)$ uses $\|\nabla F(\vw_t)\| \leq F(\vw_t)$ from Lemma \ref{lem:obj_grad_ub}
and $(ii)$ uses the inductive hypothesis. So we can apply Lemma
\ref{lem:stable_descent}:
\begin{align}
    F(\vw_{t+1}) - F(\vw_t) &\leq \langle \nabla F(\vw_t), \vw_{t+1} - \vw_t \rangle + 4 F(\vw_t) \|\vw_{t+1} - \vw_t\|^2 \\
    &= - \eta \|\nabla F(\vw_t)\|^2 + 4 \eta^2 F(\vw_t) \|\nabla F(\vw_t)\|^2 \\
    &= - \eta \left( 1 - 4 \eta F(\vw_t) \right) \|\nabla F(\vw_t)\|^2 \\
    &\Eqmark{i}{\leq} - \frac{1}{2} \eta \|\nabla F(\vw_t)\|^2 \\
    &\Eqmark{ii}{\leq} - \frac{1}{8} \eta \gamma^2 F(\vw_t)^2, \label{eq:stable_inter_1}
\end{align}
where $(i)$ uses the inductive hypothesis, and $(ii)$ uses Lemma \ref{lem:obj_grad_lb}.
Note that the condition of Lemma \ref{lem:obj_grad_lb} is satisfied here, since
$F(\vw_t) \leq 1/8\eta \leq \log(2)/n$, which implies for every $i \in [n]$ that
\begin{align}
    \log(1 + \exp(-a_t^i)) &= n \cdot \frac{1}{n} \log(1 + \exp(-a_t^i)) \leq n \cdot \frac{1}{n} \sum_{j=1}^n \log(1 + \exp(-a_t^j)) \\
    &= n F(\vw_t) \Eqmark{i}{\leq} n/(8 \eta) \Eqmark{ii}{\leq} \log 2,
\end{align}
where $(i)$ uses the inductive hypothesis and $(ii)$ uses $\eta \geq \eta_0 \geq n$, and
therefore $a_t^i \geq 0$.

By \Eqref{eq:stable_inter_1}, we have $F(\vw_{t+1}) < F(\vw_t)$, which completes the
induction. We can then use \Eqref{eq:stable_inter_1} to get a convergence rate of
$F(\vw_t)$:
\begin{align}
    F(\vw_{t+1}) - F(\vw_t) &\leq -\frac{1}{8} \eta \gamma^2 F(\vw_t)^2 \\
    \frac{1}{F(\vw_t)} - \frac{1}{F(\vw_{t+1})} &\leq -\frac{1}{8} \eta \gamma^2 \frac{F(\vw_t)}{F(\vw_{t+1})} \\
    \frac{1}{F(\vw_{t+1})} &\geq \frac{1}{F(\vw_t)} + \frac{1}{8} \eta \gamma^2 \frac{F(\vw_t)}{F(\vw_{t+1})} \\
    \frac{1}{F(\vw_{t+1})} &\geq \frac{1}{F(\vw_t)} + \frac{1}{8} \eta \gamma^2,
\end{align}
and unrolling back to $t = \tau$,
\begin{align}
    \frac{1}{F(\vw_t)} &\geq \frac{1}{F(\vw_{\tau})} + \frac{1}{8} \eta \gamma^2 (t-\tau) \\
    F(\vw_t) &\leq \frac{8}{1/F(\vw_{\tau}) + \eta \gamma^2 (t-\tau)} \\
    F(\vw_t) &\leq \frac{8}{\eta \gamma^2 (t-\tau)}
\end{align}
\end{proof}

\section{Deferred Proofs from Section \ref{sec:lower_bound}} \label{app:lower_bound}

Recall the definition $\eta_1 := \max \left\{ n, 32/\gamma^2 \log(3/\gamma)
\right\}$. Throughout the proof of Theorem \ref{thm:lower_bound}, we require
$\eta \geq \eta_1$.

\lemclassifylb*

\begin{proof}
Recall the dataset definition:
\begin{align}
    \vx_1 &= (\gamma, -\gamma) \\
    \vx_i &= (\gamma, \sqrt{1-\gamma^2}), \quad \text{for } i \in \{2, \ldots, n\},
\end{align}
The iterate $\vw_1$ after the first step can be computed exactly:
\begin{align}
    \vw_1 &= \vw_0 - \eta \nabla F(\vw_0) \\
    &= \vw_0 + \frac{\eta}{n} \sum_{i=1}^n \frac{\vx_i}{\exp(\langle \vw_0, \vx_i \rangle)+1} \\
    &= \frac{\eta}{2n} \sum_{i=1}^n \vx_i \\
    &= \frac{1}{2} \eta \gamma \ve_1 + \frac{1}{2} \eta \left( \left( 1 - \frac{1}{n} \right) \sqrt{1-\gamma^2} - \frac{\gamma}{n} \right) \ve_2.
\end{align}
The loss for each data point is determined by $\langle \vw_t, \vx_i \rangle$, which for $t=1$ can be computed exactly:
\begin{align}
    a_1^1 &= \langle \vw_1, \vx_1 \rangle \\
    &= \frac{1}{2} \eta \gamma^2 - \frac{1}{2} \eta \gamma \left( \left( 1 - \frac{1}{n} \right) \sqrt{1-\gamma^2} - \frac{\gamma}{n} \right),
\end{align}
and for $i \geq 2$,
\begin{align}
    a_1^i &= \langle \vw_1, \vx_i \rangle \\
    &= \frac{1}{2} \eta \gamma^2 + \frac{1}{2} \eta \sqrt{1-\gamma^2} \left( \left( 1 - \frac{1}{n} \right) \sqrt{1-\gamma^2} - \frac{\gamma}{n} \right).
\end{align}
Note that
\begin{align}
    \left( \left( 1 - \frac{1}{n} \right) \sqrt{1-\gamma^2} - \frac{\gamma}{n} \right) &\geq \frac{5}{6} \sqrt{1-(1/4)^2} - \frac{1}{6} \geq \frac{1}{2}.
\end{align}
since $\gamma \leq 1/4$ and $n \geq 6$. So
\begin{equation} \label{eq:classify_lb_inter}
    a_1^1 \leq \frac{1}{2} \eta \gamma^2 - \frac{1}{4} \eta \gamma \leq -\frac{1}{8} \eta \gamma.
\end{equation}

Let $t_c$ be the first timestep where $a_t^1 > 0$. Then for all $t$ with $1 \leq t <
t_c$,
\begin{align}
    a_{t+1}^1 - a_t^1 &= \langle \vw_{t+1} - \vw_t, \vx_1 \rangle \\
    &= \frac{\eta}{n} \sum_{j=1}^n \frac{\langle \vx_j, \vx_1 \rangle}{\exp(a_t^j) + 1} \\
    &= \frac{\eta}{n} \frac{\|\vx_1\|^2}{\exp(a_t^1) + 1} + \frac{\eta}{n} \sum_{j \geq 2} \frac{\langle \vx_j, \vx_1 \rangle}{\exp(a_t^j) + 1} \\
    &= \frac{2 \eta \gamma^2}{n} \frac{1}{\exp(a_t^1) + 1} + \frac{\eta (\gamma^2 - \gamma \sqrt{1-\gamma^2})}{n} \sum_{j \geq 2} \frac{1}{\exp(a_t^j) + 1} \\
    &\leq \frac{2 \eta \gamma^2}{n} \frac{1}{\exp(a_t^1) + 1} - \frac{\eta \gamma}{2n} \sum_{j \geq 2} \frac{1}{\exp(a_t^j) + 1} \\
    &\leq \frac{2 \eta \gamma^2}{n} - \frac{\eta \gamma}{2n} \sum_{j \geq 2} \frac{1}{\exp(a_t^j) + 1} \\
    &\leq \frac{2 \eta \gamma^2}{n}. \label{eq:classify_lb_inter_2}
\end{align}
Now we can recurse over $t$ from $1$ to $t_c$:
\begin{align}
    a_{t_c}^1 &= a_1^1 + \sum_{t=1}^{t_c-1} (a_{t+1}^1 - a_t^1) \\
    0 &\Eqmark{i}{<} -\eta \gamma/8 + \sum_{t=1}^{t_c-1} (a_{t+1}^1 - a_t^1) \\
    \eta \gamma/8 &\Eqmark{ii}{\leq} 2(t_c-1) \eta \gamma^2/n \\
    t_c &\geq 1 + \frac{n}{16 \gamma},
\end{align}
where $(i)$ uses \Eqref{eq:classify_lb_inter}, and $(ii)$ uses
\Eqref{eq:classify_lb_inter_2}.
\end{proof}

\lemstablelb*

\begin{proof}
Denote $F_i(\vw) = \log(1 + \exp(-\langle \vw, \vx_i \rangle))$, so that $F =
\frac{1}{n} \sum_{i=1}^n F_i$.

\paragraph{Step 1: Constructing the dataset} Let $\vv_1 = (\gamma, -\delta)$ and $\vv_2
= (\gamma, \sqrt{1-\gamma^2})$, where $\delta \in [0, \sqrt{1-\gamma^2}]$. We define a
dataset as
\begin{equation}
    \vx_i = \begin{cases}
        \vv_1 & i \leq k := \lceil n/2 \rceil \\
        \vv_2 & i > k
    \end{cases}
\end{equation}
It is easy to verify that this dataset satisfies Assumption \ref{ass:dataset} and has
maximum margin $\gamma$ with $\vw_* = (1, 0)$. Denoting
\begin{equation}
    \lambda = \log \left( \frac{1}{\exp(2n/k\eta) - 1} \right),
\end{equation}
we want to choose $\delta$ so that $\langle \vw_1, \vv_1 \rangle = \lambda - 1$ and
$\langle \vw_1, \vv_2 \rangle \geq \Omega(\eta)$. Notice that $\log(1+\exp(-\lambda)) =
2n/k\eta$, so if for some $\vw$ we have $\langle \vw, \vv_1 \rangle < \lambda$, then
\begin{align}
    F(\vw) &= \frac{1}{n} \sum_{i=1}^n \log(1 + \exp(-\langle \vw, \vx_i \rangle)) \\
    &= \frac{k}{n} \log(1 + \exp(-\langle \vw, \vv_1 \rangle)) + \frac{n-k}{n} \log(1 + \exp(-\langle \vw, \vv_2 \rangle)) \\
    &> \frac{k}{n} \log(1 + \exp(-\langle \vw, \vv_1 \rangle)) \\
    &> \frac{k}{n} \log(1 + \exp(-\lambda)) \\
    &> \frac{k}{n} \frac{2n}{k \eta} = 2/\eta. \label{eq:lb_obj_cond}
\end{align}

We can find a $\delta$ that satisfies these conditions with a simple derivation. First,
we write $\vw_1$ in closed form:
\begin{equation}
    \vw_1 = - \eta \nabla F(\vzero) = \frac{\eta}{2n} \sum_{i=1}^n \vx_i = \frac{\eta k}{2n} \vv_1 + \frac{\eta (n-k)}{2n} \vv_2
\end{equation}
We want $\delta$ to satisfy:
\begin{align}
    \langle \vw_1, \vv_1 \rangle &= \lambda - 1 \\
    \left\langle \frac{\eta k}{2n} \vv_1 + \frac{\eta (n-k)}{2n} \vv_2, \vv_1 \right\rangle &= \lambda - 1 \\
    \frac{\eta k}{2n} \left( \|\vv_1\|^2 + (n/k-1) \langle \vv_1, \vv_2 \rangle \right) &= \lambda - 1 \\
    \frac{\eta k}{2n} \left( \gamma^2 + \delta^2 + (n/k-1) (\gamma^2 - \delta \sqrt{1-\gamma^2})) \right) &= \lambda - 1 \\
    \delta^2 - (n/k-1) \sqrt{1-\gamma^2} \delta + (n/k) \gamma^2 - (2n/\eta k)(\lambda - 1) &= 0 \label{eq:delta_quad_cond}
\end{align}
This equation has a solution, since the discriminant is
\begin{align}
    (n/k-1)^2 (1-\gamma^2) - 4((n/k) \gamma^2 - (2n/\eta k)(\lambda-1)) &\Eqmark{i}{\geq} (n/k-1)^2 (1-\gamma^2) - 4(n/k) \gamma^2 \\
    &\Eqmark{ii}{\geq} \frac{1}{4} (1-\gamma^2) - 8 \gamma^2 \\
    &= \frac{1}{4} - \frac{33}{4} \gamma^2 \\
    &\Eqmark{iii}{\geq} \frac{1}{4} \left( 1  - \frac{33}{36} \right) > 0,
\end{align}
where $(i)$ uses that $\lambda > 1$, which comes from
\begin{equation}
    \lambda = \log \left( \frac{1}{\exp(2n/k\eta) - 1} \right) \geq \log \left( \frac{1}{\exp(4/\eta) - 1} \right) \geq \log \left( \frac{1}{\exp(4/1152) - 1} \right) \geq 1,
\end{equation}
where the last inequality uses $\eta \geq \eta_1 \geq 32/\gamma^2 \geq 1152$, $(ii)$
uses $k/n \in [1/2, 2/3]$, and $(iii)$ uses $\gamma \leq 1/6$. We will let $\delta_*$ be
the smaller of the two solutions to \Eqref{eq:delta_quad_cond}, that is,
\begin{align}
    \delta_* &= \frac{1}{2} \left( (n/k-1) \sqrt{1-\gamma^2} - \sqrt{(n/k-1)^2 (1-\gamma^2) - 4((n/k) \gamma^2 - (2n/\eta k) (\lambda - 1)} \right) \\
    &= \frac{1}{2} \frac{(n/k-1)^2 (1-\gamma^2) - \left( (n/k-1)^2 (1-\gamma^2) - 4((n/k) \gamma^2 - (2n/\eta k) (\lambda - 1) \right)}{(n/k-1) \sqrt{1-\gamma^2} + \sqrt{(n/k-1)^2 (1-\gamma^2) - 4((n/k) \gamma^2 - (2n/\eta k) (\lambda - 1)}} \\
    &= 2 \frac{(n/k) \gamma^2 - (2n/\eta k) (\lambda - 1) }{(n/k-1) \sqrt{1-\gamma^2} + \sqrt{(n/k-1)^2 (1-\gamma^2) - 4((n/k) \gamma^2 - (2n/\eta k) (\lambda - 1)}}
\end{align}
Notice that
\begin{equation} \label{eq:delta_ub}
    \delta_* \leq 2 \frac{(n/k) \gamma^2}{(n/k-1) \sqrt{1-\gamma^2}} \Eqmark{i}{\leq} \frac{6 \gamma^2}{\sqrt{1 - \gamma^2}} \Eqmark{ii}{\leq} 7 \gamma^2
\end{equation}
where $(i)$ uses $n/k \geq 3/2$ and $(ii)$ uses $\gamma \leq 1/6$. To lower bound
$\delta_*$, we'll need the following fact:
\begin{align}
    \log(k \eta/2n) &\Eqmark{i}{\leq} \log(\eta/3) \Eqmark{ii}{\leq} \log(9/\gamma^2) + \frac{\gamma^2}{9} \left( \frac{\eta}{3} - \frac{9}{\gamma^2} \right) \\
    &\leq 2 \log(3/\gamma) + \frac{1}{27} \eta \gamma^2 \Eqmark{iii}{\leq} \frac{1}{8} \eta \gamma^2 + \frac{1}{27} \eta \gamma^2 \leq \frac{1}{6} \eta \gamma^2,
\end{align}
where $(i)$ uses $k/n \leq 2/3$, $(ii)$ uses concavity of $\log$, and $(iii)$ uses $\eta
\geq 16/\gamma^2 \log(3/\gamma)$. Now we can upper bound $\lambda$ as
\begin{align} \label{eq:lambda_ub}
    \lambda &= \log \left( \frac{1}{\exp(2n/k\eta) - 1} \right) \Eqmark{i}{\leq} \log(k\eta/2n) \leq \frac{1}{6} \eta \gamma^2,
\end{align}
where $(i)$ uses $\exp(z) \geq 1+z$. Finally, we can lower bound $\delta_*$:
\begin{align} \label{eq:delta_lb}
    \delta_* &\geq 2 \frac{(n/k) \gamma^2 - (2n/\eta k) \lambda }{(n/k-1) \sqrt{1-\gamma^2} + \sqrt{(n/k-1)^2 (1-\gamma^2) - 4((n/k) \gamma^2 - (2n/\eta k) \lambda}} \\
    &\Eqmark{i}{\geq} 2 \frac{(n/k) \gamma^2 - (n/3k) \gamma^2 }{(n/k-1) \sqrt{1-\gamma^2} + \sqrt{(n/k-1)^2 (1-\gamma^2) - 4((n/k) \gamma^2 - (n/3k) \gamma^2}} \\
    &= \frac{4}{3} \frac{(n/k) \gamma^2}{(n/k-1) \sqrt{1-\gamma^2} + \sqrt{(n/k-1)^2 (1-\gamma^2) - (8/3) (n/k) \gamma^2}} \\
    &\geq \frac{2}{3} \frac{(n/k) \gamma^2}{(n/k-1) \sqrt{1-\gamma^2}} \Eqmark{ii}{\geq} \frac{4}{3} \gamma^2,
\end{align}
where $(i)$ uses \Eqref{eq:lambda_ub} and $(ii)$ uses $n/k \geq 3/2$ and $\gamma > 0$.
Together, \Eqref{eq:delta_ub} and \Eqref{eq:delta_lb} show that $\delta_* =
\Theta(\gamma^2)$, which we will use in the following analysis. For the remainder of the
proof, we will choose $\delta = \delta_*$, yielding the dataset $\vv_1 = (\gamma,
-\delta_*), \vv_2 = (\gamma, \sqrt{1-\gamma^2})$.

\paragraph{Step 2: $\langle \vw_t, \vv_1 \rangle$ is increasing}
Let $t_s = \min \{t \geq 0 \;|\; \langle \vw_t, \vv_1 \rangle \geq \lambda\}$. We want
to show that $\langle \vw_t, \vv_1 \rangle$ is increasing from $t=1$ to $t=t_s$. For $t
< t_s$,
\begin{align}
    \langle \vw_{t+1}, \vv_1 \rangle - \langle \vw_t, \vv_1 \rangle &= \frac{\eta}{n} \sum_{i=1}^n \frac{\langle \vx_i, \vv_1 \rangle}{\exp(\langle \vw_t, \vx_i \rangle) + 1} \\
    &= \frac{k \eta}{n} \left( \frac{\|\vv_1\|^2}{\exp(\langle \vw_t, \vv_1 \rangle) + 1} + \frac{(n/k-1) \langle \vv_1, \vv_2 \rangle}{\exp(\langle \vw_t, \vv_2 \rangle) + 1} \right) \\
    &\Eqmark{i}{\geq} \frac{k \eta}{n} \left( \frac{\gamma^2}{\exp(\langle \vw_t, \vv_1 \rangle) + 1} - \frac{6(n/k-1) \gamma^2}{\exp(\langle \vw_t, \vv_2 \rangle) + 1} \right) \\
    &\geq \frac{k \eta \gamma^2}{n} \left( \frac{1}{\exp(\lambda) + 1} - \frac{6}{\exp(\langle \vw_t, \vv_2 \rangle) + 1} \right),
\end{align}
where $(i)$ uses
\begin{equation} \label{eq:lb_negative_dot}
    \langle \vv_1, \vv_2 \rangle = \gamma^2 - \delta_* \sqrt{1-\gamma^2} \geq (1 - 7 \sqrt{1-\gamma^2}) \gamma^2 \geq -6 \gamma^2.
\end{equation}
This means that, to show $\langle \vw_{t+1}, \vv_1 \rangle \geq \langle \vw_t, \vv_1
\rangle$, it suffices that
\begin{align}
    \frac{1}{\exp(\lambda) + 1} &\geq \frac{6}{\exp(\langle \vw_t, \vv_2 \rangle) + 1} \\
    (\iff) \quad \exp(\lambda) + 1 &\leq \frac{1}{6} \left( \exp(\langle \vw_t, \vv_2 \rangle) + 1 \right) \\
    (\impliedby) \quad 2 \exp(\lambda) &\leq \frac{1}{6} \exp(\langle \vw_t, \vv_2 \rangle) \\
    (\iff) \quad \lambda + \log(12) &\leq \langle \vw_t, \vv_2 \rangle. \label{eq:lb_inc_cond}
\end{align}
To lower bound $\langle \vw_t, \vv_2 \rangle$, we use $\langle \vw_t, \vv_1 \rangle \leq
\lambda$ to bound $\langle \vw_t, \ve_2 \rangle$:
\begin{align}
    \langle \vw_t, \vv_1 \rangle &< \lambda \\
    \gamma \langle \vw_t, \ve_1 \rangle - \delta_* \langle \vw_t, \ve_2 \rangle &< \lambda \\
    \langle \vw_t, \ve_2 \rangle &> \frac{1}{\delta_*} \left( \gamma \langle \vw_t, \ve_1 \rangle - \lambda \right) \\
    \langle \vw_t, \ve_2 \rangle &> \frac{1}{\delta_*} \left( \gamma \langle \vw_t, \vw_* \rangle - \lambda \right) \\
    \langle \vw_t, \ve_2 \rangle &\Eqmark{i}{>} \frac{1}{\delta_*} \left( \gamma \langle \vw_1, \vw_* \rangle - \lambda \right) \\
    \langle \vw_t, \ve_2 \rangle &> \frac{1}{\delta_*} \left( \frac{1}{2} \eta \gamma^2 - \lambda \right) \\
    \langle \vw_t, \ve_2 \rangle &\Eqmark{ii}{>} \frac{1}{3 \delta_*} \eta \gamma^2 \\
    \langle \vw_t, \ve_2 \rangle &\Eqmark{iii}{>} \frac{1}{21} \eta,
\end{align}
where $(i)$ uses the monotonicity part of Lemma \ref{lem:max_margin}, $(ii)$ uses
\Eqref{eq:lambda_ub}, and $(iii)$ uses \Eqref{eq:delta_ub}. Now we can lower bound
$\langle \vw_t, \vv_2 \rangle$:
\begin{align}
    \langle \vw_t, \vv_2 \rangle &= \langle \vw_t, \ve_1 \rangle \gamma + \langle \vw_t, \ve_2 \rangle \sqrt{1-\gamma^2} \\
    &\Eqmark{i}{\geq} \frac{1}{21} \eta \sqrt{1-\gamma^2} \\
    &\Eqmark{ii}{\geq} \frac{1}{24} \eta \\
    &= \frac{9}{240} \eta + \frac{1}{240} \eta \\
    &\Eqmark{iii}{\geq} \lambda + \log(12)
\end{align}
where $(i)$ uses $\langle \vw_t, \vw_* \rangle \geq \langle \vw_0, \vw_* \rangle = 0$
from Lemma \ref{lem:max_margin}, $(ii)$ uses $\gamma \leq 1/6$, and $(iii)$ uses
\Eqref{eq:lambda_ub}, $\gamma \leq 1/6$, and $\eta \geq \eta_1 \geq 32/\gamma^2 \geq
1152$. This proves \Eqref{eq:lb_inc_cond}, which implies that $\langle \vw_t, \vv_1
\rangle$ is increasing for $t < t_s$.

\paragraph{Step 3: Trajectory Analysis}
Recall from \Eqref{eq:lb_obj_cond} that $\langle \vw_t, \vv_1 \rangle < \lambda$ implies
$F(\vw_t) > 2/\eta$, so to prove the lemma, it suffices to show that $t_s \geq
\Omega(1/\gamma^2)$. To lower bound $t_s$, we upper bound $\langle \vw_t, \vv_1
\rangle$: for all $1 \leq t < t_s$,
\begin{align}
    \langle \vw_{t+1} , \vv_1 \rangle - \langle \vw_t, \vv_1 \rangle &= \frac{\eta}{n} \sum_{i=1}^n \frac{\langle \vx_i, \vv_1 \rangle}{\exp(\langle \vw_t, \vx_i \rangle) + 1} \\
    &= \frac{k \eta}{n} \left( \frac{\|\vv_1\|^2}{\exp(\langle \vw_t, \vv_1 \rangle) + 1} + \frac{(n/k - 1) \langle \vv_1, \vv_2 \rangle}{\exp(\langle \vw_t, \vv_2 \rangle) + 1} \right) \\
    &\Eqmark{i}{\leq} \frac{k \eta}{n} \frac{\|\vv_1\|^2}{\exp(\langle \vw_t, \vv_1 \rangle) + 1} \\
    &= \frac{k \eta}{n} \frac{\gamma^2 + \delta_*^2}{\exp(\langle \vw_t, \vv_1 \rangle) + 1} \\
    &\Eqmark{ii}{\leq} \frac{8k \eta \gamma^2}{n} \frac{1}{\exp(\langle \vw_t, \vv_1 \rangle) + 1} \\
    &\Eqmark{iii}{\leq} \frac{8k \eta \gamma^2}{n} \frac{1}{\exp(\langle \vw_1, \vv_1 \rangle) + 1} \\
    &= \frac{8k \eta \gamma^2}{n} \frac{1}{\exp(\lambda - 1) + 1} \leq \frac{8e k \eta \gamma^2}{n} (\exp(2n/k\eta) - 1) \\
    &\leq \frac{8e k \eta \gamma^2}{n} (\exp(4/\eta) - 1) \Eqmark{iv}{\leq} \frac{8e k \eta \gamma^2}{n} \frac{4/\eta}{4/1152} (\exp(4/1152) - 1) \\
    &\leq \frac{8e k}{n} 1152 (\exp(4/1152) - 1) \gamma^2 \leq 59 \gamma^2,
\end{align}
where $(i)$ uses $\langle \vv_1, \vv_2 \rangle \leq 0$ from \Eqref{eq:lb_negative_dot},
$(ii)$ uses $\delta_* \leq 7 \gamma^2$ from \Eqref{eq:delta_ub}, $(iii)$ uses that
$\langle \vw_t, \vv_1 \rangle$ is increasing for $t < t_s$, and $(iv)$ uses convexity of
$\exp$ together with $\eta \geq \eta_1 \geq 32/\gamma^2 \geq 1152$. Finally, this means
that
\begin{align}
    \langle \vw_{t_s}, \vv_1 \rangle &\leq \langle \vw_1, \vv_1 \rangle + 59 \gamma^2 (t_s-1) \\
    \lambda &\leq (\lambda - 1) + 59 \gamma^2 (t_s-1) \\
    t_s &\geq 1 + \frac{1}{59 \gamma^2}.
\end{align}
\end{proof}

\thmlowerbound*

\begin{proof}
The result follows more or less immediately from Lemmas \ref{lem:classify_lb} and
\ref{lem:stable_lb}. Recalling
\begin{align}
    t_c &= \min \left\{ t \geq 0 : \langle \vw_t, \vx_i \rangle \geq 0 \text{ for all } i \in [n] \right\} \\
    t_s &= \min \left\{ t \geq 0 : F(\vw_t) \leq 2/\eta \right\},
\end{align}
We know that $\tau \geq t_c$, since
\begin{equation}
    \langle \vw_t, \vx_i \rangle \leq 0 \implies F(\vw_t) \geq \frac{1}{n} \log(1 + \exp(-\langle \vw_t, \vx_i \rangle)) \geq \frac{\log 2}{n} \geq \frac{1}{8 \eta},
\end{equation}
where the last line uses $\eta \geq \eta_1 \geq n$. Also, $\tau \geq t_s$ is
immediate from definitions.

The only detail needing consideration is the condition $n \geq 6$ for Lemma
\ref{lem:classify_lb}. If this condition is not met, then $n \leq 5 \leq
1/\gamma$, so Lemma \ref{lem:stable_lb} implies
\begin{equation}
    \tau \geq t_s \geq \frac{1}{59 \gamma^2} = \frac{1}{118 \gamma^2} + \frac{1}{118 \gamma^2} \geq \frac{1}{118 \gamma^2} + \frac{n}{118 \gamma}.
\end{equation}
If $n \geq 6$, then Lemmas \ref{lem:classify_lb} and \ref{lem:stable_lb} imply
\begin{equation}
    \tau \geq \max \left\{ t_c, t_s \right\} \geq \max \left\{ \frac{n}{16 \gamma}, \frac{1}{59 \gamma^2} \right\} \geq \frac{1}{2} \left( \frac{n}{16 \gamma} + \frac{1}{59 \gamma^2} \right) = \frac{n}{32 \gamma} + \frac{1}{118 \gamma^2}.
\end{equation}
In all cases, we have $\tau \geq (n/\gamma + 1/\gamma^2)/118$.
\end{proof}

\section{Auxiliary Lemmas}

\begin{lemma} \label{lem:pot_lb}
Denote $G(\vw) = \frac{1}{n} \sum_{i=1}^n \frac{1}{\exp(\langle \vw, \vx_i \rangle) +
1}$. If $F(\vw) \geq c$, then $G(\vw) \geq \frac{1 - \exp(-nc)}{n}$. If additionally $c
\leq 1/n$, then $G(\vw) \geq c/2$.
\end{lemma}

\begin{proof}
We want to lower bound the solution of the following:
\begin{align}
    \inf_{\vw \in \mathbb{R}^d} &~\frac{1}{n} \sum_{i=1}^n \frac{1}{\exp(\langle \vw, \vx_i \rangle) + 1} \label{eq:pot_lb_inter_1} \\
    \text{s.t.} &~\frac{1}{n} \sum_{i=1}^n \log(1 + \exp(-\langle \vw, \vx_i \rangle)) \geq c.
\end{align}
The solution of \Eqref{eq:pot_lb_inter_1} is lower bounded by the following:
\begin{align}
    \inf_{a_1, \ldots, a_n \in \mathbb{R}} &~\frac{1}{n} \sum_{i=1}^n \frac{1}{\exp(a_i) + 1} \label{eq:pot_lb_inter_2} \\
    \text{s.t.} &~\frac{1}{n} \sum_{i=1}^n \log(1 + \exp(-a_i)) \geq c.
\end{align}
Changing variables to $\ell_i = \log(1 + \exp(-a_i))$, we have
\begin{equation}
    \frac{1}{\exp(a_i) + 1} = 1 - \exp(-\ell_i),
\end{equation}
so \Eqref{eq:pot_lb_inter_2} can be rewritten as
\begin{align}
    \inf_{\ell_1, \ldots, \ell_n \geq 0} &~\frac{1}{n} \sum_{i=1}^n (1 - \exp(-\ell_i)) \\
    \text{s.t.} &~\frac{1}{n} \sum_{i=1}^n \ell_i \geq c,
\end{align}
or $G(\vw) \geq 1 - \phi$, where $\phi$ is the solution of
\begin{align}
    \sup_{\ell_1, \ldots, \ell_n \geq 0} &~\frac{1}{n} \sum_{i=1}^n \exp(-\ell_i) \label{eq:pot_lb_inter_3} \\
    \text{s.t.} &~\frac{1}{n} \sum_{i=1}^n \ell_i \geq c.
\end{align}
Note that \Eqref{eq:pot_lb_inter_3} is equivalent to
\begin{align}
    \sup_{\ell_1, \ldots, \ell_n \geq 0} &~\frac{1}{n} \sum_{i=1}^n \exp(-\ell_i) \label{eq:pot_lb_inter_4} \\
    \text{s.t.} &~\frac{1}{n} \sum_{i=1}^n \ell_i = c,
\end{align}
since the supremum of \Eqref{eq:pot_lb_inter_3} will not be achieved when $\frac{1}{n}
\sum_{i=1}^n \ell_i > c$. Now, the supremum of \Eqref{eq:pot_lb_inter_4} is achieved by
$\ell_1 = cn$ and $\ell_i = 0$ for $i \geq 2$, and this is shown by Karamata's
inequality (Lemma \ref{lem:karamatas}): for any other $\ell_1' \geq \ldots \geq \ell_n'$
with $\frac{1}{n} \sum_{i=1}^n \ell_i' = c$, all conditions of Lemma \ref{lem:karamatas}
are satisfied, so that $\frac{1}{n} \sum_{i=1}^n \exp(-\ell_i) \geq \frac{1}{n}
\sum_{i=1}^n \exp(-\ell_i')$. Therefore $\phi = 1 - \frac{1 - \exp(-cn)}{n}$, and
\begin{equation}
    G(\vw) \geq 1 - \phi = \frac{1 - \exp(-cn)}{n},
\end{equation}
which is exactly the desired conclusion.

If we additionally have $c \leq 1/n$, then
\begin{align}
    G(\vw) \geq \frac{1 - \exp(-cn)}{n} \Eqmark{i}{\geq} cn \frac{1 - \exp(-1)}{n} + (1 - cn) \frac{1 - \exp(0)}{n} = c(1 - 1/e) \geq c/2,
\end{align}
where $(i)$ uses concavity of $-\exp(\cdot)$.
\end{proof}

\begin{lemma}[Karamata's Inequality, Theorem 1 \citep{kadelburg2005inequalities}] \label{lem:karamatas}
If $f: \mathbb{R} \rightarrow \mathbb{R}$ is convex, and $a_1, \ldots, a_n, b_1, \ldots,
b_n \in \mathbb{R}$ are such that
\begin{enumerate}
    \item $a_1 \geq \ldots \geq a_n$ and $b_1 \geq \ldots \geq b_n$,
    \item $a_1 + \ldots + a_i \geq b_1 + \ldots + b_i$ for every $i \leq n$,
    \item $a_1 + \ldots + a_n = b_1 + \ldots + b_n$,
\end{enumerate}
then $f(a_1) + \ldots + f(a_n) \geq f(b_1) + \ldots + f(b_n)$.
\end{lemma}

\begin{lemma}[Lemma 4.5 from \citet{crawshaw2025constant}] \label{lem:stable_descent}
For $\vw, \vw' \in \mathbb{R}^d$, if $\|\vw - \vw'\| \leq 1$ then
\begin{equation}
    F(\vw') \leq F(\vw) + \langle \nabla F(\vw), \vw' - \vw \rangle + 4 F(\vw) \|\vw - \vw'\|^2.
\end{equation}
\end{lemma}

\begin{lemma}[Lemma 25 from \citet{crawshaw2025local}] \label{lem:obj_grad_ub}
$\|\nabla F(\vw)\| \leq F(\vw)$ for all $\vw \in \mathbb{R}^d$.
\end{lemma}

\begin{lemma}[Lemma 26 from \citet{crawshaw2025local}] \label{lem:obj_grad_lb}
For all $\vw \in \mathbb{R}^d$, if $\langle \vw, \vx_i \rangle \geq 0$ for all $i \in
[n]$, then
\begin{equation}
    \|\nabla F(\vw)\| \geq \frac{\gamma}{2} F(\vw).
\end{equation}
\end{lemma}

\begin{lemma} \label{lem:obj_pot_lb}
If $\langle \vw, \vx_i \rangle \geq 0$ for all $i \in [n]$, then
\begin{equation}
    F(\vw) \leq 2 G(\vw).
\end{equation}
\end{lemma}

Note: This lemma appears as an intermediate step in the analysis of \citet{wu2024large}.
We include it here for completeness.

\begin{proof}
\begin{align}
    F(\vw) &= \frac{1}{n} \sum_{i=1}^n \log(1 + \exp(-\langle \vw, \vx_i \rangle)) \leq \frac{1}{n} \sum_{i=1}^n \exp(-\langle \vw, \vx_i \rangle) \\
    &\leq \frac{1}{n} \sum_{i=1}^n \frac{1}{\exp(\langle \vw, \vx_i \rangle)} \Eqmark{i}{\leq} \frac{1}{n} \sum_{i=1}^n \frac{2}{\exp(\langle \vw, \vx_i \rangle) + 1} \\
    &\leq 2 G(\vw),
\end{align}
where $(i)$ uses $\langle \vw, \vx_i \rangle \geq 0$.
\end{proof}

\end{document}